\documentclass{article}

\usepackage{arxiv}

\usepackage{amsmath}
\usepackage{amssymb}
\usepackage{amsthm}
\usepackage{bbm} 
\usepackage{stmaryrd}

\usepackage{natbib}
\setcitestyle{authoryear,open={(},close={)}}
\usepackage[ruled]{algorithm2e}
\makeatletter
\newcounter{framework}
\newenvironment{framework}[1][htb]
  {
   \let\c@algocf\c@framework
   \begin{algorithm}[#1]%
  }{\end{algorithm}}

\usepackage{thm-restate}

\newcommand{\R}{\mathbb{R}}
\newcommand{\E}{\mathbb{E}}
\newcommand{\X}{\mathcal{X}}
\newcommand{\Y}{\mathcal{Y}}
\newcommand{\Z}{\mathcal{Z}}
\newcommand{\calH}{\mathcal{H}}
\newcommand{\cD}{c_\Delta}
\newcommand{\G}{\mathcal{G}}
\newcommand{\F}{\mathcal{F}}
\newcommand{\deff}{d_{\text{eff}}}

\usepackage[utf8]{inputenc} 
\usepackage[T1]{fontenc}    
\usepackage{hyperref}       
\usepackage{url}            
\usepackage{booktabs}       
\usepackage{amsfonts}       
\usepackage{nicefrac}       
\usepackage{microtype}      
\usepackage{xcolor}         
\usepackage{wrapfig}
\usepackage{enumitem}

\title{Structured Prediction in Online Learning}

\author{%
  Pierre Boudart \\
  INRIA, École Normale Supérieure\\
  CNRS, PSL Research University\\
  Paris, France\\
  \texttt{pierre.boudart@inria.fr} \\
  \And
  Alessandro Rudi \\
  INRIA, École Normale Supérieure\\
  CNRS, PSL Research University\\
  Paris, France\\
  \texttt{alessandro.rudi@inria.fr}
  \AND
  Pierre Gaillard \\
  Univ. Grenoble Alpes, Inria, \\
  CNRS, Grenoble INP, LJK \\
  Grenoble, France\\
  \texttt{pierre.gaillard@inria.fr} \\
}

\begin{document}

\maketitle

\begin{abstract}
    We study a theoretical and algorithmic framework for structured prediction in the online learning setting. 
    The problem of structured prediction, i.e. estimating function where the output space lacks a vectorial structure, is well studied in the literature of supervised statistical learning. 
    We show that our algorithm is a generalisation of optimal algorithms from the supervised learning setting, and achieves the same excess risk upper bound also when data are not i.i.d.
    Moreover, we consider a second algorithm designed especially for non-stationary data distributions, including adversarial data. We bound its stochastic regret in function of the variation of the data distributions. 
\end{abstract}

\section{Introduction}\label{section:introduction}
Online learning is a subfield of statistical learning in which a learner receives a flow of data generated by an environment \citep{cesa2006prediction, orabona, hazan2023introduction}. The learner has to learn from the flow of data, and adapt to the data which could be non-stationary or adversarial.
More formally, at each time step \(t\), the learner receives a context \(x_t\in\X\) from which he makes a prediction \(\hat{z_t} = f_t(x_t) \in\Z\). His prediction is then compared to the true label \(y_t\in\Y\), which is observed. The learner than pays an error \(\Delta(\hat z_t, y_t)\) measured by a known loss function \(\Delta : \Z \times \Y \to \R\). The goal of the learner is to minimise his regret
\begin{equation}
    R_T = \sum_{t=1}^T \Delta(\hat{z_t}, y_t) - \Delta(f^*_t(x_t), y_t) \,,
\end{equation}
where \(\smash{f^*_t(x_t) \in \mathrm{arg\,min}_{z \in \Z} \Delta(z,y_t)}\). 
The inputs \(x_t\) and labels \(y_t\) are generated sequentially by the environment and could be adversarial. This could model the change of behaviour of a customer or an evolution of the environment such as climate change. Note that in our framework, unlike the standard regret definition in online learning, the learner's performance is compared to the best function $f_t$ at each round, similar to the approach used in dynamic regret~\citep{herbster1998tracking}. 

When the output space contains a vectorial structure, statistical learning provides many algorithms with statistical guarantees. However more and more applications involve an output space which lacks a linear structure, such as translation \citep{lacoste2006word}, image segmentation \citep{forsyth2002computer}, protein folding \citep{joachims2009predicting}, ranking \citep{duchi2010consistency}.
These problems are often referred as structured prediction problems, because the output space may be represented for instance as a sequence, a graph, or an ordered set.
In practice, an ad hoc method is designed to solved each of these problems and is most of the time based on surrogate methods and empirical risk minimisation. If they achieve good results in practice, they however lack generalisation and are not built in order to have good theoretical guarantees.

We consider the structured prediction framework of \textit{Implicit Loss Embedding (ILE)} \citep{ILE}, in which the loss is of the form $\smash{\Delta(z,y) = \langle \psi(z),\varphi(y)\rangle}$ for some unknown and infinite dimensional feature maps $\psi:\mathcal{Z}\to \mathcal{H}$ and $\varphi:\mathcal{Y}\to \mathcal{H}$ into an unknown RKHS $\mathcal{H}$ (see Definition~\ref{def:ILE}). Such an assumption is satisfied by most losses for rich enough feature maps and used in the practical applications detailed above. 
\citep{ILE} study this framework in a statistical supervised learning setting and provide a general algorithm for general problems including discrete outputs and manifold regression. Their algorithm comes with statistical guarantees on the excess risk when data are i.i.d. only. 

In the context of prediction of arbitrary sequences, the closer works to ours are \citep{mcmahan2014unconstrained} and \citep{pacchiano2018online}. On the one hand, \citep{mcmahan2014unconstrained} analyses a loss written as an inner product in a Hilbert space $\Delta(z,y) = \langle z, \varphi(y)\rangle$. However, they assume that the action space $\mathcal{Z}$ is itself a Hilbert space \(\calH\) which thus has a vectorial structure, contrary to the setting we consider. On the other hand, \cite{pacchiano2018online} also considers a loss expressed by a kernel with full information and partial feedback, but they do not consider contextual information $x_t$ and require prior knowledge of the kernel feature maps $\psi$ and $\varphi$, which we do not need. 

\paragraph{Contributions}
Our work is the first to study structured prediction in the framework of prediction of arbitrary sequences. 

We first introduce a new algorithm, called \textit{OSKAAR} (Algorithm~\ref{algo:first algo}) and inspired by the work of \citep{ILE} in the statitical framework. Given a RKHS $\mathcal{G}$ from $\mathcal{X}$ to $\mathcal{H}$ associated to a kernel of the feature space $k:\mathcal{X}\times\mathcal{X} \to \mathbb{R}$ and a regularization parameter $\lambda$, \textit{OSKAAR} achieves a regret upper-bound (Theorem~\ref{theorem:FirstRegretBound}) of order\footnote{The symbol $\lesssim$ is a rough inequality that neglects contants and logarithmic factors.}:
\begin{equation*}
    R_T \lesssim \sqrt{T \big(\deff(\lambda) + \smash{\min_{g\in\G} L_T(g) \big)} } \,, \qquad \text{where}  \quad L_T(g) := \sum_{t=1}^T || g(x_t) - \varphi(y_t)||^2_\calH + \lambda ||g||^2_\G \,,
\end{equation*}
where  \(\deff(\lambda)\) is the effective dimension~\eqref{eq:effective_dimension} that measures the size of the RKHS and \(L_T(g)\) measures how the RKHS $g$ is able to interpolate features of the data. In particular, if there is a function \(g^*\in\G\) that perfectly models the features \((\varphi(y_t))_t\), i.e. \( g^*(x_t) = \varphi(y_t)\) for all $t$, noting that $\deff(\lambda) \lesssim T/\lambda$, the above result yields a regret bound of the order of \(\smash{O(T^{3/4})}\). However, such an assumption is strong even for i.i.d. data, and the above bound might be linear in $T$ in the worst-case. in the worst-case scenario. This is not surprising, as the learner's performance is compared to the best possible baseline \(\mathrm{arg\,min}_{z \in \Z} \Delta(z,y_t)\) at each time step, which is generally unattainable.

To weaken the above assumption, we also prove the following expected regret bound for \textit{OSKAAR} (Theorem~\ref{theorem:expectancy regret bound}):
\begin{equation*}
    \E [ R_T ] \lesssim  \sqrt{T \big( \deff(\lambda)  +\smash{ \min_{g \in \mathcal{G}} \bar L_T(g)} }, \quad \text{where} \quad \bar L_T(g) = \E \bigg[ \sum_{t=1}^T \lVert g(x_t) - \E [\varphi(y_t)|x_t] \rVert^2_\calH \bigg] +  \lambda \lVert g \rVert^2_\G \,,
\end{equation*}
where the expectation is taken with respect to the possible randomness of the data $(x_t,y_t)$. In the context of arbitrary sequences, the two above results exactly match. Yet, the assumption that there exists some $g^*$ such that $g^*(x_t) = \E [\varphi(y_t)|x_t]$  for all $t$, is much weaker in general than assuming $g^*(x_t) = \varphi(y_t)$ since random variation of $\varphi(y_t)$ are not considered. Such an assumption is weak in the i.i.d. statistical framework and standard in the analysis of Kernel Ridge Regression \citep{caponnetto2007optimal, steinwart2008support}. It  corresponds to assuming that the data distribution lies in the RKHS. In particular, we show that our analysis allows to recover (up to a log factor) the optimal rate of \cite{ILE} in the i.i.d. setting, by designing an estimator $\bar f_T$ that satisfies the excess risk upper-bound:
\begin{equation*}
    \E_{x,y} [\Delta(\bar f_T(x), y) - \Delta(f^*(x), y)] \lesssim  T^{-1/4}  + T^{-1/2} \sqrt{\log (\delta^{-1})} \qquad \text{w.p. } \ 1-\delta.
\end{equation*} 
Our estimator $\bar f_T$ is constructed via a careful online to batch conversion to face with two challenges: the loss $\Delta(z,y)$ being non-convex in $z$ standard online to batch conversion techniques that use $\smash{\bar f_T  = \sum_{t=1}^T f_t}$ are not possible here; our result holds with high-probability which is challenging to obtain with such techniques \citep{vanderhoeven2023highprobability}. 

The above result still hold under the assumption that $g^*(x_t) = \E [\varphi(y_t)|x_t]$  for all $t$ for some $g^*$, which is weak for stationary data but strong in our framework of arbitrary sequences. Our third contribution aims at relaxing this assumption. We design a second algorithm, referred to as \textit{SALAMI} (Algorithm~\ref{algo:non stationary aggregation}), that achieves under the assumption that there exists $g_t^* \in \mathcal{G}$ such that  $g_t^*(x_t) = \E[\varphi(y_t)|x_t]$ for all $t$:
\begin{equation}
    \E [ R_T ] = 
        \left\{
        \begin{array}{ll}
         \tilde O(V_\G^{1/6} T^{5/6}) & \text{if } \lambda = V_\G^{-1/3} T^{1/3} \\
         \tilde O(V_0^{1/4} T^{3/4}) & \text{if } \lambda = V_0^{-1/2} T^{1/2} 
        \end{array}
        \right. \,,
\end{equation}
where $V_0$ and $V_\G$ are two different measures of the non-stationarity of the sequence $(g_t^*)$:
\[
    V_0 = 1 + \sum_{t=2}^T \mathbbm 1 \{g_t^* \neq g_{t-1}^* \}  \qquad \text{and} \qquad  V_\G := \|g_1^*\|_\G + \sum_{t=2}^T \lVert g_t^* - g_{t-1}^* \rVert_\G  \,.
\]

\paragraph{Paper outline}
In the next section, we recall the setting of the problem and the background on the \textit{ILE} definition. In Section \ref{section:algorithm}, we introduce our first algorithm \textit{Online Structured prediction with Kernel Aggregating Algorithm Regression (OSKAAR)} and the algorithm from the batch setting. In section \ref{section:first bound}, we bound the regret of our algorithm. In Section \ref{section:expectancy regret}, we recover the convergence rate from the batch setting without stochastic assumption. And in Section \ref{section:dynamic regret}, we introduce our second algorithm \textit{Structured prediction ALgorithm with Aggregating MIxture (SALAMI)} for non-stationary data and bound its stochastic regret.
The details of the proofs can be found in appendix.
Moreover, in Appendix \ref{section:appendix stochastic regret bounds} and \ref{section:appendix non stationary high probability} we provide bounds in high probability for both the stationary and the non-stationary settings.

\section{Problem Setting and Background}\label{section:problem setting background}
We recall the setting and introduce the main notations used throughout the paper. We then discuss the limitations of the previous works.
We denote by \( \X,\Y \) and \(\Z\)  respectively the input, label and output spaces of the learning problem. We denote by \( \Delta : \Z \times \Y \to \R \) the loss function, which measures the error between a prediction in \(\Z\) and a true label in \(\Y\). Having two different spaces \(\Y\) and \(\Z\) allows to consider applications where the outputs do not match the labels such as ranking \citep{duchi2010consistency}.

\begin{minipage}{.53\textwidth} 
\paragraph{Online Learning Framework}

Our online framework is formalised as a game between a learner and an environment, see Framework \ref{algo:OLCI}. At each time step \(t\ge 1\), the user receives a context \(x_t \in\X\), computes a prediction \(\hat{z_t} = f_t(x_t) \in\Z\) based on the current context \(x_t\) and the history \( (x_1, y_1, \dots, x_{t-1}, y_{t-1})\). The true label \(y_t\in\Y\) is then revealed to the learner, which incurs a loss \(\Delta(\hat z_t, y_t)\).
In this framework we are in the full information setting. That is to say that observing the label \(y_t\) enables the learner to compute the loss \(\Delta(z, y_t)\) for all \(z\in\Z\).
\end{minipage}\quad 
\begin{minipage}{.43\textwidth}
\begin{framework}[H]\label{algo:OLCI}
\caption{Online learning framework with contextual information}
\For{Each time step \(t\) in \(1 \dots T\)}{
    Get information \(x_t\in\X\)\\
    Compute the prediction \\
    \qquad \(\hat{z_t} = f_t(x_t) \in\Z\)\\
    Observe the label \(y_t\in\Y\)\\
    Get loss \(\Delta(\hat{z_t}, y_t)\in\R\) \\
    Update predictor \(f_{t+1}\)
}
\end{framework}
\end{minipage}

The online learning setting allows us to also work with adversarial or non-stationary data, i.e. data that are not i.i.d. This could model a change of the environment. Throughout the paper we consider a loss \( \Delta : \Z \times \Y \to \R \) that admits an \textit{Implicit Loss Embedding (ILE)}, see Definition \ref{def:ILE}, with feature maps \(\psi, \varphi\), and a Hilbert space \(\calH\).
\begin{restatable}[ILE \citep{ILE}]{defi}{DefILE}
\label{def:ILE}
    A continuous map \(\Delta : \Z \times \Y \to \R\) is said to admit an Implicit Loss Embedding (ILE) if there exists a separable Hilbert space \(\calH\) and two measurable bounded maps \(\psi : \Z \to \calH\) and \(\varphi : \Y \to \calH\), such that for any \(z\in\Z\) and \(y\in\Y\) we have 
    \begin{equation}
        \Delta(z,y) = \langle \psi(z) , \varphi(y) \rangle_\calH
    \end{equation}
    and \(||\varphi(y)||_\calH \le 1\). Additionally, we define \(\cD = \sup_{z\in\Z} ||\psi(z)||_\calH\).
\end{restatable}

In particular we do not assume that the loss is convex or differentiable.
The metric to evaluate the performance of a learning algorithm is the regret defined as
\begin{equation}
    R_T = \sum_{t=1}^T \Delta(\hat{z_t}, y_t) - \Delta(f^*_t(x_t), y_t)
\end{equation}
where \(f^*_t(x_t) \in \mathrm{arg\,min}_{z \in \Z} \Delta(z,y_t)\) is used as the baseline. Taking the optimum inside the sum as we do is stronger than taking the optimum of the sum as is usually done.

\paragraph{Structured Prediction}
This is the most general setting in supervised learning. We say that a learning problem is structured if we have one of the following conditions \citep{nowak}:
\begin{itemize}[nosep]
    \item The loss is different than the 0-1 loss : \( \Delta(z,y) \neq \mathbbm{1}[z\neq y] \).
    \item The size of the output space is exponentially larger than the natural dimension of the output elements.
\end{itemize}
The first condition implies that some pairs of outputs and labels are closer than others. For instance, two sets that differ by only one element should be closer to each other compared to sets with an empty intersection. The second condition characterizes a space of sequences, where the cardinality is exponential in the size of the dictionary used to build the sequences. The following spaces and losses are structured:
\begin{itemize}[nosep]
    \item Subsets of \( \llbracket k \rrbracket := \{ 1, \dots, k \} \) with the negative F1 score \( \Delta(z,y) = - 2 |z \cap y | / (|z| + |y|) \)
    \item Ordered elements: \( \Z=\Y= ( \llbracket k \rrbracket, < ) \) with \( \Delta(z,y) = |z-y| \)
    \item Sequences of \(k\) elements of a dictionary $\mathcal{D}$ with the Hamming distance \( \Delta(z,y) = \lVert z-y \rVert_0 \)
    \item Ranking, Information Retrieval: 
    the goal is to predict an ordered list of documents or web pages from \( x\in\X \) a query in a search engine. The output space \(\Z\) is the space of permutations and the label space \(\Y\) contains scalar scores representing the relevance of each document for the query \citep{duchi2010consistency}.
\end{itemize}
Note that we do not assume to have a vectorial structure in the output or the label space.

\paragraph{Standard Approach}
The classical learning approach, in the supervised learning setting, is Empirical Risk Minimisation (ERM) \citep{devroye2013probabilistic}. The expected risk is estimated by the empirical risk, and \( f_n \) computed as its minimiser. The underlying idea is that \(f_n\) should approach \(f^*\) as size of the sample \(n\) grows. The estimator \(f_n\) is defined as follows
\begin{equation}
    f_n = \arg\min_{f\in\F} \dfrac{1}{n} \sum_{i=1}^n \Delta(f(x_i), y_i)
\end{equation}
where \(\F\) is a class of function and an hyper-parameter of the method. 
When the loss \(\Delta\) is convex and the output space \(\Z\) has a vectorial structure ERM becomes an efficient strategy for a large family of spaces \(\F\).
However this strategy presents some limitations \citep{ILE}:
\begin{itemize}[nosep]
    \item \textbf{Modeling.} If we do not assume to have a vectorial structure on the output space \(\Z\), it is not clear how to design a suitable function space \(\F\). For instance, given \( f_1, f_2 : \X \to \Z \), there is no guarantee that \(f_1 + f_2\) takes values in \(\Z\) as well.
    \item \textbf{Computations.} If the function space \(\F\) is non-linear or the loss in non-convex, solving ERM can be challenging. Most approaches, such as gradient descent, are based on the regularity of the loss or the optimisation domain.
\end{itemize}

\paragraph{Existing results in the batch statistical framework}
We briefly recall the main results from \cite{ILE}.
The authors introduced the \textit{ILE} assumption (see Def.~\ref{def:ILE}) and studied learning problems that satisfy this definition in the supervised learning setting. The mathematical constructs introduced in this definition, such as the feature maps \(\psi, \varphi\) and the Hilbert space \(\calH\), are used solely for analysis purposes and algorithm design. Notably, they are not required for making predictions. An important feature of their work, which we also achieve, is that our online algorithms do not need prior knowledge of \(\psi, \varphi\) and \(\calH\).

Let \((x_i, y_i)_{i=1}^n\) be a sample of i.i.d. data. \cite{ILE} consider the ERM estimator \(g_n : \X \to \calH\) that learns the features \(\varphi(y)\) as
\begin{equation}\label{eq:ERM gn (0)}
    g_n := \arg\min_{g\in\G} \dfrac{1}{n} \sum_{i=1}^n \lVert \varphi(y_i) - g(x_i) \rVert^2_\calH + \lambda \lVert g \rVert^2_\G
\end{equation}
over a known kernel space \(\G\). Choosing a kernel space gives us a closed form solution and strong algebraic properties to analyse the algorithm. Moving the problem to the feature space \(\calH\), enables us to enjoy the vectorial structure of \(\calH\). The authors then define the predictor \(f_n : \X \to \Z\)  as an optimisation problem using \(g_n\) as follows
\begin{equation}
    f_n(x) := \arg\min_{z\in\Z} \langle \psi(z), g_n(x) \rangle_\calH = \arg\min_{z\in\Z} \sum_{i=1}^n \alpha_i(x) \Delta(z, y_i) \,,
\end{equation}
where $\alpha_i$ are coefficients obtained by resorting to the representer theorem.
Let \(\mathcal{E}\) and \(\mathcal{R}\) be the expected risk of \(f_n : \X \to \Z\) and \(g_n : \X \to \calH\) respectively and \(f^*\) and \(g^*\) be their respective minimizers. \cite{ILE} show that the excess risk of \(f_n\) is controlled by the one of \(g_n\) enabling them to carry out their analysis. The following comparison inequality and convergence rate are derived:
\begin{equation*}
    \mathcal{E}(f) - \mathcal{E}(f^*) \lesssim \sqrt{\mathcal{R}(g) - \mathcal{R}(g^*)}     \le O\left( n^{-1/4} \log\left(\delta^{-1}\right) \right) \quad \text{w.p.}\  1-\delta
\end{equation*}
where \(\lesssim\) does not take into account multiplicative constants independent of \(n\) and \(\delta\).

\paragraph{Limitations of previous works}
This work is limited to the batch statistical framework with i.i.d. data. However some applications involve a flow of data; or data generated by non-stationary distributions including adversarial data. Our work is the first to study structured prediction in the setting of arbitrary sequences. 

\section{A General Algorithm for Online Structured Prediction}
In this section we introduce our algorithm \textit{OSKAAR (Online Structured prediction with Kernel Aggregating Algorithm Regression)} and bound its regret.

\subsection{Introducing our Algorithm: \textit{OSKAAR}}\label{section:algorithm}

To simplify notations, we may denote \(\varphi(y_t)\) by \(\varphi_t\in\calH\). We recall that the feature map \( \varphi : \Y \to \calH \) is constant over time, the index \(t\) in this notation denotes only the variation of \(y_t\) over time.

\begin{algorithm}[H]
\caption{\textit{OSKAAR -- Online Structured prediction with Kernel Aggregating Algorithm Regression}}
\label{algo:first algo}
\KwIn{\(\lambda >0\), kernel \( k: \X \times \X \to \R \)}
\For{Each time step \(t\) in \(1 \dots T\)}{
    Get information \(x_t\in\X\)\\
    Update \( \smash{\beta^t (x) = (K_t + \lambda I )^{-1} v_t(x)} \) where \( K_t \) and \(v_t\) are defined after Eq. \eqref{eq:beta t} \\
    \(\smash{\hat{z_t} = \arg\min_{z\in\Z} \langle\psi(z), \hat g_t(x_t) \rangle_\calH = \arg\min_{z\in\Z} \sum_{s=1}^{t-1} \beta^t_s(x_t) \Delta(z, y_s) }\)\\
    Observe ground truth \(y_t\in\Y\)\\
    Get loss \(\Delta(\hat{z_t}, y_t)\in\R\)
}
\end{algorithm}

We introduce our first algorithm, see Algorithm \ref{algo:first algo}, which is inspired by the learning procedure of \cite{ILE}. However, we use a variant of Kernel Ridge Regression that has a different regularisation which is crucial in the context of arbitrary data, \textit{Kernel
Aggregating Algorithm Regression (KAAR)}, see \cite{gammerman2012online, jezequel_efficient_2019}. At each time step \(t\in\llbracket T \rrbracket\), we compute \(\hat g_t : \X\to\calH \) as follows
\begin{equation}\label{eq:hat g_t}
    \hat g_t := \arg\min_{g\in\G} \sum_{s=1}^{t-1} \lVert g(x_s) - \varphi_s \rVert^2_\calH + \lambda \lVert g \rVert^2_\G + \lVert g(x_t) \rVert^2_\calH \,, 
\end{equation}
where \( \G \) is a vRKHS with feature map \(\phi\) such that \( \smash{\sup_{x\in\X} \lVert \phi(x) \rVert \le \kappa < \infty} \), see Appendix \ref{section:first bound appendix} for more details.
And \(f_t\) is defined as an optimisation problem with respect to \(\hat g_t\) as in the batch setting
\begin{equation}\label{eq:f_t}
    f_t(x) := \arg\min_{z\in\Z} \langle \psi(z), \hat g_t(x) \rangle_\calH = \arg\min_{z\in\Z} \sum_{s=1}^{t-1} \beta^t_s(x) \Delta(z, y_s) \,,
\end{equation}
where the coefficients \(\beta^t_s\) come from the representer theorem, and are defined as follows
\begin{equation}\label{eq:beta t}
    \beta^t (x) = (K_t + \lambda I )^{-1} v_t(x)
\end{equation}
with \(\smash{K_t \in \R^{t\times t}} \) the Gram matrix defined by \(\smash{(K_t)_{i,j} = k(x_i, x_j)}\), and \(\smash{v_t(x) \in \R^t}\) defined by \(\smash{ (v_t(x))_s = k(x,x_s) }\).
Thus, at each time step \(t\), the prediction is computed by
\begin{equation}
    \hat z_t := f_t(x_t) = \arg\min_{z\in\Z} \sum_{s=1}^{t-1} \beta^t_s(x_t) \Delta(z, y_s) .
\end{equation}
Hence, we note that, as in the supervised learning setting, the mathematical objects \(\psi, \varphi, \calH\) introduced in the definition of ILE are not needed to make a prediction. We only need the knowledge of the different labels \(y_s\) in order to compute \(\Delta(., y_s)\).

\subsection{Regret Bound of \textit{OSKAAR}}\label{section:first bound}

We start our analysis by proving a comparison inequality, see Lemma \ref{theorem:online CI}. It extends any bound on the empirical risk of \( (\hat g_t)_t \) to a bound on the regret of \( (f_t)_t \). We can therefore carry out the analysis on \((\hat g_t)_t\) for which we have a closed form solution and lies in a space with algebraic assumptions.

\begin{restatable}[Online Comparison Inequality]{lmm}{OnlineCI}\label{theorem:online CI}
    Let \( (f_t)_t \) and \( (\hat g_t)_t \) be defined as in \eqref{eq:f_t} and \eqref{eq:hat g_t} respectively. Then we have
    \begin{equation}
        R_T \le 2 \cD \sqrt{T} \sqrt{\sum_{t=1}^T ||\varphi_t - \hat{g}_t(x_t)||^2_\calH } \,.
    \end{equation}
\end{restatable}

Compared to \cite{ILE}, our online comparison inequality does not provide an upper bound with respect to a global minimiser \(\smash{g^*}\). The baseline is not reflected in the right hand side of this inequality. Finding a comparison inequality which controls the regret with respect to a baseline in the same class of functions than our estimators is left for future works.
However, this result still shifts the problem into the feature space which possesses a lot more algebraic properties. It allows us to derive the regret bound of Theorem \ref{theorem:FirstRegretBound}.
The regret bound is expressed with respect to the  effective dimension \( \deff(\lambda) \) \citep{rudi2016more, zadorozhnyi2021online} defined by 
\begin{equation}
    \label{eq:effective_dimension}
    \deff(\lambda) := \mathrm{Tr}(K (K+\lambda I)^{-1}) \quad \forall \lambda > 0
\end{equation}
where \(\smash{K \in \R^{T\times T}} \) is the Gram matrix at time \(T\). The effective dimension measures the complexity of the underlying RKHS based on a given data sample. It is a decreasing function of the scale parameter \(\lambda\) and \( \deff(\lambda) \to 0 \) when \( \lambda \to \infty \). And when \( \lambda \to 0 \) it converges to the rank of \(K\). Moreover it is always upper bounded by \( \smash{\deff(\lambda) \le \kappa^2 T / \lambda} \).
We obtain the following regret bound.
\begin{restatable}[Regret Bound of \textit{OSKAAR}]{thm}{FirstRegretBound}\label{theorem:FirstRegretBound}
    Let \( (f_t)_t \) be defined as in \eqref{eq:f_t}. Then for all \(\lambda>0\) and \(T\ge1\) we have 
    \begin{equation}
        R_T \le 2 \cD \sqrt{T} \sqrt{ \log\left(e + \dfrac{e\kappa^2 T}{\lambda}\right) \deff(\lambda) + \min_{g\in\G} L_T(g)}
    \end{equation}
    where \( \smash{L_t(g) := \sum_{s=1}^t || \varphi_s - g(x_s) ||^2_\calH + \lambda ||g||^2_\G} \) for every \( g\in\G \).
\end{restatable}

The proof of this statement is postponed to Appendix \ref{section:first bound appendix}.
In the worst case scenario, if \(\varphi_t\) is a Rademacher variable, \(\smash{\min_{g\in\G} L_T(g)}\) is linear in \(T\) yields a linear regret bound. This is due to the fact that we are comparing our model to the best possible \(z\in\Z\) at each time step, which is much too rich and linear regret is unavoidable in the worst case.
On the other hand, if there is a function \(g^*\in\G\) that perfectly models the features \(\smash{(\varphi_t)_t}\), i.e. \( \smash{g^*(x_t) = \varphi_t}\) for all $t$, by taking \( \smash{\lambda = \sqrt{T} / \lVert g^* \rVert_\G} \) and bounding the effective dimension by \( \deff(\lambda) \le \kappa^2 T / \lambda \), we obtain $\smash{R_T \le \tilde O \left( T^{3/4} \right)}$.
However, the assumption \( \smash{\sum_{t=1}^T \lVert g^*(x_t) - \varphi_t \rVert^2_\calH = 0} \) is too strong for adversarial data and even for i.i.d. data with white noise. These considerations motivate the study of the expected regret in the next section and the cumulative risk in Appendix~\ref{section:appendix stochastic regret bounds}.

\paragraph{Computation time}
At each time step \(t\), we need to compute the vector \(\beta^t(x_t) \in \R^t \). Thus the per round complexity is of \( O(t^2) \). If the kernel satisfies the capacity condition \( \smash{ \deff(\lambda) \le (T/\lambda)^\beta } \) for \( \smash{\beta\in [0,1]}\) (see Appendix \ref{section:appendix capacity condition}), using a method based on Nyström approximation \citep{jezequel_efficient_2019}, it is possible to recover the same regret with a computational complexity of \( O( \deff(\lambda)^{4/(1-\lambda)} ) \).

\section{Stochastic Regret Bounds}\label{section:expectancy regret}
In this section, we generalize the results from the supervised learning setting in \cite{ILE}. We achieve the same convergence rate as in the batch statistical framework, although our results hold without stochastic assumptions.
We are now interested in bounding the expected regret \( \E [ R_T ] \), where the expectation is taken over the possible randomness of the data \( (x_1, y_1, \dots, x_T, y_T) \). Note that the data are still generated sequentially and can be adapted to the player, in particular they can be adversarial and follow Dirac distributions.
Note that the data are still generated sequentially and can adapt to the player, meaning they can be adversarial and follow Dirac distributions. Taking the expectation helps to avoid the noise inherent in the data and enables us to obtain results closer to \cite{ILE} by replacing \(\varphi_t\) with to \( \E[\varphi_t|x_t] \) in our result. We study the same algorithm (\textit{OSKAAR}) as in the previous section, see Algorithm \ref{algo:first algo}.
We obtain the following regret bound and its corollary proved in Appendix \ref{section:appendix expectancy regret}.

\begin{restatable}[Expected Regret Bound]{thm}{ThmExpectancyRegretBound}
    \label{theorem:expectancy regret bound}
    Let \( (f_t)_t \) be defined as in \eqref{eq:f_t}. Then, for any \(g^*\in\G\), \(\lambda>0\) and \(T\ge1\), we have 
    \begin{equation}
        \E [ R_T ] \le 2 \cD \sqrt{T} \sqrt{ \deff(\lambda) \log\left(e + \tfrac{e\kappa^2T}{\lambda} \right) + \lambda \lVert g^* \rVert^2_\G + \E \left[ \sum_{t=1}^T \lVert g^*(x_t) - \E [\varphi_t|x_t] \rVert^2_\calH \right] } .
    \end{equation}
\end{restatable}


\begin{restatable}[Expected Regret Bound]{crl}{ThmExpectancyRegretBoundApplied}
    \label{theorem:expectancy regret bound applied}
    With the same assumptions than Theorem \ref{theorem:expectancy regret bound}, with \(\lambda = \sqrt{T}\). Assume that there exists \(g^*\in\G\) such that \(\smash{\E \big[ \sum_{t=1}^T \lVert g^*(x_t) - \E [\varphi_t|x_t] \rVert^2_\calH \big] = 0}\). Then, we have
    \begin{equation}
        \E [ R_T ] \le 2 \cD T^{3/4} \sqrt{\kappa^2 \log\left(e+e\kappa^2\sqrt{T}\right) + \lVert g^* \rVert^2_\G} = O \left( T^{3/4} \sqrt{\log T} \right) .
    \end{equation}
\end{restatable}

Our assumption on \(g^*\) is similar to the one done to obtain the convergence rate in \cite{ILE}. It is a common assumption in Kernel Ridge Regression theory \citep{caponnetto2007optimal, steinwart2008support}. We are assuming that there exists a function \(g^*\in \G\), such that for all \(t\in\llbracket T \rrbracket\) we have \( g^*(x_t) = \E [ \varphi_t | x_t ] \). That it to say, \(g^*\) interpolates the expectations of the data.
Up to the log factor, we retrieve the same bound in the online setting as in the supervised setting, and without assuming that the data are i.i.d. Specifically, we make no assumption on the \(x_t\), generalizing existing results that assume i.i.d. inputs. 

\paragraph{High probability regret bound}
In Theorem \ref{theorem:expectancy regret bound} and Corollary \ref{theorem:expectancy regret bound applied}, we bound the expectation of the regret. The expectation is taken over the whole data, including the whole history at each time step. Moreover a bound in expectation does not necessarily imply a bound in high probability.
It is however possible to obtain a bound in high probability on the cumulative risk, defined as
\begin{equation}
    \sum_{t=1}^T \E_{y_t} [\Delta(f_t (x_t), y_t) - \Delta(f^*_t(x_t), y_t)] \,,
\end{equation}
where at each round the expectation is taken with respect to the randomness of the next output $y_t$ only and not with all past data $x_1,y_1,\dots, x_{t-1},y_{t-1},x_t$. Computing such bounds requires more recent mathematical tools such as \cite{vanderhoeven2023highprobability}. In Appendix \ref{section:appendix stochastic regret bounds}, we prove a regret bound in high probability using a slightly different estimator. 

\paragraph{High probability excess risk bound}
In Appendix \ref{section:appendix aggregation}, we demonstrate that, when data are i.i.d., our previous results enable the design of a batch estimator \(\smash{\bar f_T}\) from the online predictors, yielding the following bound on the excess risk. With probability \(1-\delta\)
\begin{equation*}
    \E_{x,y} [\Delta(\bar f_T(x), y) - \Delta(f^*(x), y)] \le O\left( T^{-1/4} \sqrt{\log(T)} + T^{-1/2} \sqrt{\log (\delta^{-1})} \right) .
\end{equation*}
A standard online to batch conversion would have aggregated the predictors \( (f_t)_t \) by setting $\smash{\bar f_t = \sum_{t=1}^T f_t}$. However, this is not possible here because the output space \(\Z\) is not convex. To design $\bar f_T$, we thus aggregate the feature estimators \( \smash{(\hat g_t)_t} \) into a unique function \( \bar g_T \), which is used to construct $\bar f_T$. This construction requires recent technical tools \citep{vanderhoeven2023highprobability}. To sum up, our algorithm generalizes the supervised learning setting, achieving the same convergence rate up to a log factor. Moreover, our algorithm can learn from a stream of data, allowing sequential updates as data arrive step by step, instead of relying on a batch of data available from the start.

\section{Non-Stationary Online Structured Prediction}\label{section:dynamic regret}

In this section, we introduce \textit{SALAMI (Structured prediction ALgorithm with Aggregating MIxture)}, see Algorithm \ref{algo:non stationary aggregation}, an algorithm designed to handle non-stationary data distributions, including adversarial data. The non-stationarity we consider is on \( \smash{x \mapsto \E [ \varphi_t | x ]} \) rather than on the baseline \((f^*_t)_t\), which is already non-stationary throughout the paper.
We compare the feature predictors \( (\hat g_t)_t \) to a non-stationary baseline \((g_t^*)_t \in \G^T\). This approach allows us to address data with a changing distribution over time, including adversarial data. Note that we handle general data distributions, including Dirac distributions.
As the data distributions change, earlier data may become outdated. Therefore, we need to modify our previous predictor, which considers all past data equally. We treat predictors with different starting times as experts and use an expert selection algorithm to create a mixture of them. See Algorithm \ref{algo:non stationary aggregation} for details.

In the previous section, we assume the existence of some fixed function $g^* \in \mathcal{G}$ such that \( \smash{ g^*(x_t) = \E [\varphi_t|x_t] }\) for all $t$, which is weak when \( \smash{x \mapsto \E [ \varphi_t | x ]} \) is stationary. However it is not satisfied when the data distribution is non-stationary or even arbitrary. In this section, we assume that for each time step \( t\in\llbracket T \rrbracket \), there exists a function \( g_t^* \in \G \) such that \( g_t^*(x_t) = \E[\varphi_t | x_t] \). This is a very weak assumption, as we can choose a different function \(g_t^*\) for each time step.

\subsection{Regret Bound}

In order to bound the regret, we define two quantities that measure the non-stationarity of the sequence $(g_t^*)$: the continuous variation \(V_\G\) and the discrete variation \(V_0\) defined as follows
\begin{equation}
    V_\G := \|g_1^*\|_\G + \sum_{t=2}^T \lVert g_t^* - g^*_{t-1} \rVert_\G \qquad \text{and}\qquad  V_0 := 1+ \sum_{t=2}^T \mathbbm 1 [ g_t^* \neq g^*_{t-1} ] .
\end{equation}

We obtain the following regret bound.
\begin{restatable}[Expected Regret in a Non-Stationary Environment]{thm}{ThmExpectedRegretNonStationaryEnvironment}
    \label{theorem:expected regret non stationary environment}
     Assume that there exists $(g_t^*)$ a sequence in $\G$ such that \( \smash{\E \big[ \sum_{t=1}^T \lVert g_t^*(x_t) - \E[\varphi_t | x_t] \rVert^2_\calH \big] = 0} \). Then, Algorithm~\ref{algo:non stationary aggregation} run with $\lambda >0$ and \(\smash{\eta = 1 /2(\kappa \sup \lVert g \rVert_\G + 1)^2}\) satisfies 
    \begin{equation}
        \E [ R_T ] = 
        \left\{
        \begin{array}{ll}
         \tilde O(V_\G^{1/6} T^{5/6}) & \text{if } \lambda = V_\G^{-1/3} T^{1/3} \\
         \tilde O(V_0^{1/4} T^{3/4}) & \text{if } \lambda = V_0^{-1/2} T^{1/2} 
        \end{array}
        \right. \,.
    \end{equation}
\end{restatable}

In Appendix \ref{section:appendix non stationary}, we prove this statement and precise the constants and log terms in the bounds.
Therefore our method obtains a sublinear regret with variations \(V_\G, V_0\) up to \(T\). As expected, we obtain a loss of performance when facing non-stationary data distributions compared to the stationary case in Theorem \ref{theorem:expectancy regret bound applied}.
With discrete distribution changes, the rate $\smash{\tilde O(T^{3/4})}$ is unchanged compared to the stationary setting as soon as the number of changes remains constant. 

\paragraph{Calibration of $\lambda$} A limitation of Theorem~\ref{theorem:expected regret non stationary environment} is the required knowledge of $V_0$ or $V_\G$ to tune the learning rate $\lambda >0$. First, note that setting $\smash{\lambda = \Omega(T^{1/3})}$ always yield a regret of order $\smash{\tilde O(T^{5/6})}$, but at the cost of a worse dependence on the variation $V_\G$. Second, our algorithm can be easily adapted to calibrate $\lambda$ automatically, by combining experts $\smash{\hat g_{s:t}^{(\lambda)}}$ (see Eq.~\ref{eq:hat gt restart} and the algorithm details in the next section), indexed by both the starting time $s$ and an hyperparameter $\lambda$, with $\lambda$ chosen from a logarithmic finite grid.

\paragraph{Refined regret bounds under the capacity condition} A standard assumption when learning on RKHS is the capacity condition that assumes the existence of some \( \beta \in [0,1] \) and \( Q>0 \) for which \(\smash{ \deff(\lambda) \le Q (T/\lambda)^{\beta} }\) for all \(\lambda>0\). This assumption is weak since it is always verified for $\beta = 1$ but smaller values of $\beta$ yield improved computational complexity for the algorithm (see for instance~\cite{jezequel_efficient_2019}) and improved regret guarantees. We further discuss this assumption and provide a refined regret bound in Appendix \ref{section:appendix capacity condition}. In the particular case of the Gaussian kernel, the effective dimension statisfies \( \smash{\deff(\lambda) \le \left( \log\left( {T}/{\lambda} \right) \right)^d} \) \citep{altschuler2019massively}, where \(d\) is the dimension of the input space \(\X\). In this case, our result leads to the regret bound 
\begin{equation}
    \E [ R_T ] = \tilde O \big( T^{3/4} V_\G^{1/4} (\lambda+1)^{1/4} \big) \,,
\end{equation}
which improves the generic rate of Theorem~\ref{theorem:expected regret non stationary environment} from $\tilde O(T^{5/6})$ to $\tilde O(T^{3/4})$. In this case, the extension to non-stationarity comes at no cost in the regret rate as soon as $V_\G$ does not grow with time.  More details are provided in Appendix \ref{section:appendix capacity condition}.

\subsection{Algorithm Design}

We detail below our algorithm \textit{SALAMI}. Let $(g_t^*)$ be the unknown sequence satisfying the assumption of Theorem~\ref{theorem:expected regret non stationary environment}. We start from the observation that if one could identify breaking times $\smash{(t_i)_{1\leq i\leq T}}$ at which the sequence $g_t^*$ changes, one could restart \textit{OSKAAR} (Algorithm~\ref{algo:first algo}) at each $t_i$, considering that $g_t^*$ is fixed from $t_i$ to $t_{i+1} - 1$. The high-level idea of \textit{SALAMI} is to learn these restart times through a meta-aggregation procedure that combines estimators $\hat g_{s:t}$ of the sequence $g_t^*$, indexed by $s=1,\dots, t$, each assuming that $(g_t^*)$ is fixed from $s$ to $t$; defined by following the \textit{KAAR} estimator starting in \(s\)
\begin{equation}\label{eq:hat gt restart}
    \hat g_{s:t} := \arg\min_{g\in\G} \sum_{\tau =s}^{t-1} \lVert \varphi_\tau - g(x_\tau) \rVert_\calH^2 + \lambda \lVert g \rVert_\G^2 + \lVert g(x_t) \rVert^2_\calH.
\end{equation}
For each \( t\in \llbracket T \rrbracket \), the feature predictor \(\hat g_t \) of \textit{SALAMI} is then defined as a convex combination of $\hat g_{s:t}$ for $1\leq s\leq t$. Formally, \textit{SALAMI} learns a probability vector \( p_t \in \Delta_T \) and defines
\begin{equation}\label{eq:hat gt dynamic}
    \hat g_t = \sum_{s=1}^T p_t(s) \hat g_{s:t} .
\end{equation}
The next part of the algorithm is how to choose the weights $p_t(s)$. To do so, this is done by using the exponentially weighted average forecaster (EWA) \(w_t \in \Delta_T \), which needs a small adaptation to deal with the fact that $\hat g_{s:t}$ only produces predictions for $t\geq s$. Following the idea of \cite{GaillardStoltzEtAl2014} for sleeping experts, this can be done by defining the auxiliary losses
\[
    \tilde \ell_t(s) =
    \begin{cases}
        \ell_t(\hat g_{s:t}) \text{ if } s\le t \\
        \ell_t(\hat g_t) \text{ if } s> t
    \end{cases} \qquad \text{where} \qquad \ell_t(g) := \lVert \varphi_t - g(x_t) \rVert^2_\calH .
\]
That is by assigning the loss of the algorithm itself $\ell_t(\hat g_t)$ to all expert that are inactive. The weights $p_t(s)$ are then defined as: $p_1(1) = 1$ and for $t> 1$:
\begin{equation}\label{eq:loss tilde}
    p_t(s) = \frac{w_t(s)}{\sum_{k=1}^t w_t(k)} \qquad \text{where} \qquad w_t(k) \propto \exp\left( -\eta\sum_{s=1}^{t-1} \tilde \ell_s(k) \right)
\end{equation}
for some learning rate $\eta >0$. Finally, \textit{SALAMI} defines the predictor \( \hat f_t : \X \to \Z \) as:
\begin{equation}\label{eq:hat ft dynamic}
    \hat f_t(x) := \arg\min_{z\in\Z} \langle \psi(z), \hat g_t(x) \rangle_\calH .
\end{equation}

\paragraph{Computational complexity} Note that in its current form, \textit{SALAMI} needs to consider an increasing number of experts $\hat g_{s:t}$ over time, which increases the per-round space and time complexities of \textit{OSKAAR} by a factor of $O(t)$. However, this problem can be addressed by using more sophisticated intervals than $[s,t]$ as done in \cite{IFLH,expert_intervals,strongly_adaptive}, which reduces the overhead in complexities to a factor of $O(\log t)$. Extending our work to such intervals is straightforward, but we have chosen to restrict ourselves to intervals $[s,t]$ to simplify the understanding of the algorithm.

\begin{algorithm}[H]
\caption{\textit{SALAMI -- Structured prediction ALgorithm with Aggregating MIxture}} 
\label{algo:non stationary aggregation}
\KwIn{\(\lambda>0\), exp-concavity constant \(\eta\) of  \((\ell_t)_t\), kernel \( k : \X \times \X \to \R \)}
\For{Each time step \(t\) in \(1 \dots T\)}{
    Get information \(x_t\in\X\)\\
    \For{Each expert \(s\) in \(1 \dots t\)}{
        Compute \( \smash{\hat g_{s:t} := \arg\min_{g\in\G} \sum_{\tau =s}^{t-1} \lVert \varphi_\tau - g(x_\tau) \rVert_\calH^2 + \lambda \lVert g \rVert_\G^2 + \lVert g(x_t) \rVert^2_\calH} \)
    }
    \For{Each expert \(s\) in \(1 \dots T\)}{
        Compute EWA \( w_t(s) \propto w_{t-1}(s) \exp(-\eta \tilde \ell_t(s)) \quad \text{where} \quad \tilde \ell_t(s) = 
        \begin{cases}
            \ell_t(\hat g_{s:t}) \text{ if } s\le t \\
            \ell_t(\hat g_t) \text{ if } s> t
        \end{cases} \) \\
        Compute \(\smash{ p_t(s) \propto
        \begin{cases}
            w_t(s) \text{ if } s\le t \\
            0 \text{ if } s> t
        \end{cases} }\)
    }
    Compute the aggregate predictor \( \smash{\hat g_t = \sum_{s=1}^T p_t(s) \hat g_{s:t}} \) \\
    Compute the prediction \( \smash{\hat z_t = \hat f_t(x_t) = \arg\min_{z\in\Z} \langle \psi(z), \hat g_t(x_t) \rangle_\calH }\) \\
    Observe ground truth \(y_t\in\Y\)\\
    Get loss \(\Delta(\hat{z_t}, y_t)\in\R\)
}
\end{algorithm}

\paragraph{Acknowledgements.}
A.R. acknowledges the support of the French government under management of Agence Nationale de la Recherche as part of the “Investissements d’avenir” program, reference ANR-19-P3IA-0001 (PRAIRIE 3IA Institute) and the support of the European Research Council (grant REAL 947908).

\bibliography{bibfile}

\newpage
\appendix

\begin{center}
    \huge
    APPENDIX
\end{center}

\section{Proof of Theorem \ref{theorem:FirstRegretBound}: Regret Bound of \textit{OSKAAR}}\label{section:first bound appendix}
We first introduce some additional notations on kernels.
Let \(k : \X \times \X \to \R\) be a positive semidefinite kernel, and \(\F = \overline{\text{span}\{k(x,.) | x\in\X \}} \) its associated RKHS. We denote by \(\phi: \X \to \F\) the feature map \(\phi(x) = k(x,.)\). We assume \(\phi\) to be bounded by \(\lVert\phi(x)\rVert_\F \le \kappa < \infty\). For more details on RKHS see \cite{aronszajn1950theory, berlinet2011reproducing}.
We may now introduce the following operators:
\begin{itemize}[nosep]
    \item \(S_t : \F \to \R^t\), s.t. \(f\in\F \mapsto (\langle \phi(x_s), f \rangle)_{s=1}^t\)
    \item \(S^*_t : \R^t \to \F\), s.t. \(v=(v_i)_{i=1}^t \mapsto \sum_{s=1}^t v_i \phi(x_i)\)
    \item \(C_t = S^*_t S_t : \F \to \F\)
    \item We have that \(C_t = \sum_{s=1}^t \phi(x_s) \otimes \phi(x_s)\)
    \item \(K_t = S_t S^*_t\) is the empirical kernel matrix
    \item \(A_\lambda = A + \lambda I\) for any symmetric linear operator \(A\), where \(I\) is the identity and \(\lambda\in\R\)
\end{itemize}
For the space of function \(\G : \X \to \calH\) we choose an vector-valued RKHS \(\G = \calH \otimes \F \) \citep{micchelli2004kernels, alvarez2012kernels}, which is a direct generalisation of scalar-valued RKHS.

We define \(L_t\) for \(t\in\llbracket T \rrbracket\) in a more general setting and rewrite it using Hilbertian operators
\[
    L_t (g, g^*_t) = \sum_{s=1}^t \lVert g^*_s - g(x_s) \rVert^2_\calH + \lambda \lVert g\rVert^2_\G = \lVert H_t \rVert^2 -2g^*S^*_t H_t + \langle g, C_{t, \lambda} g \rangle
\]
where \(H_t\) is the vector \((g^*_1, \dots , g^*_t)\in \calH^t \). And we denote by \( L_t(g) \) the application \( L_t(g, \varphi(Y_t)) \), where \(\varphi(Y_t)\) is the vector \((\varphi_1, \dots, \varphi_t) \in \calH^t\).
We define the following functions 
\begin{equation}
    g_{t+1} = \arg\min_{g\in\G} L_t(g) = C_{t,\lambda}^{-1} S^*_t \varphi(Y_t) \,,
\end{equation}
\begin{equation}
    \hat{g}_{t+1} = \arg\min_{g\in\G} L_t(g) + \lVert g(x_{t+1})\rVert^2_\calH = C_{t+1,\lambda}^{-1} S^*_t \varphi(Y_t) \,.
\end{equation}

\(\hat{g}_{t+1}\) is directly used in the definition of the algorithm, while \(g_{t+1}\) is only used in the proof of its regret.

We recall and prove the results from in Section \ref{section:first bound}.

\OnlineCI*

\begin{proof}
    We add and subtract two terms.
    \begin{align*}
        R_T = \sum_{t=1}^T &\Delta(f_t(x_t), y_t) - \Delta(f^*_t(x_t), y_t) \\
       =  \sum_{t=1}^T &\langle \psi(f_t(x_t)), \varphi_t \rangle - \langle \psi(f_t(x_t)), \hat{g}_t(x_t) \rangle \\
       +&\langle \psi(f_t(x_t)), \hat{g}_t(x_t) \rangle - \langle \psi(f^*_t(x_t)), \hat{g}_t(x_t) \rangle \\
       +&\langle \psi(f^*_t(x_t)), \hat{g}_t(x_t) \rangle - \langle \psi(f^*_t(x_t)), \varphi_t \rangle \\
       \le  \sum_{t=1}^T &\langle \psi(f_t(x_t)), \varphi_t \rangle - \langle \psi(f_t(x_t)), \hat{g}_t(x_t) \rangle \\
       +&\langle \psi(f^*_t(x_t)), \hat{g}_t(x_t) \rangle - \langle \psi(f^*_t(x_t)), \varphi_t \rangle \\
       =\sum_{t=1}^T &\langle \psi(f_t(x_t)) - \psi(f^*_t(x_t)), \varphi_t - \hat{g}_t(x_t) \rangle
    \end{align*}
    where the inequality comes from the definition of $f_t$. We now apply successively Cauchy-Schwartz and Jensen's inequalities to conclude the proof.
    \begin{align*}
        R_T &\le \sum_{t=1}^T ||\psi(f_t(x_t)) - \psi(f^*_t(x_t))||_\calH \cdot ||\varphi_t - \hat{g}_t(x_t)||_\calH \\
        &\le 2 \cD \sum_{t=1}^T ||\varphi_t - \hat{g}_t(x_t)||_\calH \\
        &\le 2 \cD \sqrt{T} \sqrt{\sum_{t=1}^T ||\varphi_t -\hat{g}_t(x_t)||^2_\calH }
    \end{align*}
\end{proof}

We bound the regret of the \textit{KAAR} estimator. We generalise the proof of \cite{jezequel_efficient_2019} to vRKHS.
\begin{restatable}[General Regret \textit{KAAR} Estimator]{lmm}{}
\label{theorem:regret vawk}
    Let \((h_t)_{t=1}^T\) be bounded vectors in \(\calH\) such that \(\lVert h_t \rVert_\calH \le B < \infty \) for all \(t\in\llbracket T \rrbracket\). Let \( H_T \in \calH^T \) be the vector \( (h_1, \dots, h_T) \).
    Let \(\lambda>0\), and let us define the \textit{KAAR} predictors as follow
    \begin{equation*}
        \hat g_t = \arg\min_{g\in \G} \sum_{s=1}^{t-1} \lVert h_s - g(x_s) \rVert^2_\calH + \lambda \lVert g \rVert^2_\G + \lVert g(x_t) \rVert^2_\calH.
    \end{equation*}
    Then we have
    \begin{equation}
        \sum_{t=1}^T \lVert h_t  - \hat g_t (x_t)  \rVert^2_\calH \le B^2 \log\left(e + \dfrac{e\kappa^2 T}{\lambda} \right) \deff(\lambda) + \min_{g\in\G} L_T(g, H_T) .
    \end{equation}
\end{restatable}

\begin{proof}
    We follow and adapt the proof of Theorem 9 from \cite{jezequel_efficient_2019} to vector RKHS without Nyström approximation. We start by adding telescopic terms.
    \begin{align*}
        \sum_{t=1}^T &\lVert h_t -\hat{g}_t(x_t)\rVert^2_\calH \\
        &= 
        \sum_{t=1}^T \lVert h_t -\hat{g}_t(x_t)\rVert^2_\calH - L_T(g_{T+1}, H_T) + L_T(g_{T+1}, H_T) \\
        &= \sum_{t=1}^T \left [ \lVert h_t -\hat{g}_t(x_t)\rVert^2_\calH + L_{t-1}(g_t, H_{t-1}) - L_t(g_{t+1}, H_t) \right ] + L_T(g_{T+1}, H_T) \,.
    \end{align*}
    Let \(Z(t)= \lVert h_t -\hat{g}_t(x_t)\rVert^2_\calH + L_{t-1}(g_t, H_{t-1}) - L_t(g_{t+1}, H_t)\), and let us study its terms separately.

    Note that \(\langle g_{t+1}, C_{t, \lambda} g_{t+1} \rangle = H_t^* S_t C_{t, \lambda}^{-1}  C_{t, \lambda} g_{t+1} = H_t^* S_t g_{t+1}\).
    
    Therefore \(L_t(g_{t+1}, H_t) = \lVert H_t \rVert^2 - \langle g_{t+1}, C_{t, \lambda} g_{t+1} \rangle\).

    Let us focus now on \( \lVert h_t -\hat{g}_t(x_t)\rVert^2 = \lVert h_t \rVert^2 -2\langle h_t, \hat{g}_t(x_t) \rangle + \langle \hat{g}_t(x_t), \hat{g}_t(x_t) \rangle \).
    Note that
    \begin{align*}
        \langle h_t, \hat{g}_t(x_t) \rangle &= \hat{g}_t^* \phi(x_t) h_t \\
        &= \hat{g}_t^* ( S^*_t H_t - S^*_{t-1} H_{t-1}) \\
        &= \hat{g}_t^* ( C_{t, \lambda} g_{t+1} - C_{t-1, \lambda} g_t ) \\
        &= \langle \hat{g}_t , C_{t, \lambda} g_{t+1} - C_{t-1, \lambda} g_t \rangle
    \end{align*}
    and
    \begin{equation*}
        \langle \hat{g}_t(x_t), \hat{g}_t(x_t) \rangle = \langle \hat{g}_t , \phi(x_t) \langle \phi(x_t), \hat{g}_t \rangle \rangle = \langle \hat{g}_t , [\phi(x_t) \otimes \phi(x_t)] \hat{g}_t \rangle = \langle \hat{g}_t, (C_{t, \lambda} - C_{t-1, \lambda}) \hat{g}_t \rangle \,.
    \end{equation*}
    Therefore we obtain
    \begin{equation*}
        \lVert h_t -\hat{g}_t(x_t)\rVert^2 = \lVert h_t \rVert^2 -2 \langle \hat{g}_t , C_{t, \lambda} g_{t+1} - C_{t-1, \lambda} g_t \rangle + \langle \hat{g}_t, (C_{t, \lambda} - C_{t-1, \lambda}) \hat{g}_t \rangle \,.
    \end{equation*}

    Thus by putting everything together, we get
    \begin{align*}
        Z(t) &= -2 \langle \hat{g}_t , C_{t, \lambda} g_{t+1} - C_{t-1, \lambda} g_t \rangle + \langle \hat{g}_t, (C_{t, \lambda} - C_{t-1, \lambda}) \hat{g}_t \rangle - \langle g_t, C_{t-1, \lambda} g_t \rangle + \langle g_{t+1}, C_{t, \lambda} g_{t+1} \rangle \\
        &= \langle \hat{g}_t - g_{t+1} , C_{t, \lambda} (\hat{g}_t - g_{t+1}) \rangle - \langle \hat{g}_t -g_t , C_{t-1, \lambda} (\hat{g}_t -g_t) \rangle \\
        &\le \langle \hat{g}_t - g_{t+1} , C_{t, \lambda} (\hat{g}_t - g_{t+1}) \rangle \,.
    \end{align*}
    Now note that we can factorise
    \begin{equation*}
        \hat{g}_t - g_{t+1} = C_{t, \lambda}^{-1} S^*_{t-1} H_{t-1} - C_{t, \lambda}^{-1} S^*_t H_t = -C_{t, \lambda}^{-1} \phi(x_t) h_t \,.
    \end{equation*}
    We thus bound \(Z(t)\) by
    \begin{equation*}
        Z(t) \le \langle \phi(x_t) h_t , C_{t, \lambda}^{-1} \phi(x_t) h_t \rangle = \lVert h_t \rVert^2 \langle \phi(x_t) , C_{t, \lambda}^{-1} \phi(x_t) \rangle \le B^2 \langle \phi(x_t) , C_{t, \lambda}^{-1} \phi(x_t) \rangle \,.
    \end{equation*}
    Finally, we proved that 
    \begin{equation*}
        \sum_{t=1}^T \lVert h_t -g(x_t)\rVert^2_\calH \le B^2 \sum_{t=1}^T \langle \phi(x_t) , C_{t, \lambda}^{-1} \phi(x_t) \rangle + L_T(g_{T+1}, H_T)\,.
    \end{equation*}
    We conclude the proof by using Propositions 1 and 2 of \cite{jezequel_efficient_2019}.
\end{proof}

We now prove our main result from Section \ref{section:first bound}.

\FirstRegretBound*

\begin{proof}
    We apply the two previous lemmas and obtain
    \begin{eqnarray*}
        R_T &=& \sum_{t=1}^T \Delta( f_t(x_t), y_t) - \Delta(f^*_t(x_t), y_t)\\
        &\stackrel{\text{(Lem.~\ref{theorem:online CI})}}{\le} & 2 \cD \sqrt{T} \sqrt{\sum_{t=1}^T ||\varphi(y_t) - \hat{g}_t(x_t)||^2_\calH} \\
        &\stackrel{\text{(Lem.~\ref{theorem:regret vawk})}}{\le }& 2  \cD \sqrt{T} \sqrt{ \log\left(e + \dfrac{e\kappa^2 T}{\lambda}\right) \deff(\lambda) + \min_{g\in\G} L_T(g)} 
    \end{eqnarray*}
    where \( B = 1 \ge \sup_{y\in\Y} \lVert \varphi(y) \rVert_\calH\).
\end{proof}

\section{Proofs of Theorem \ref{theorem:expectancy regret bound} and Corollary \ref{theorem:expectancy regret bound applied}: Stochastic Regret Bounds in Expectation}\label{section:appendix expectancy regret}
In this section, we prove the results from Section \ref{section:expectancy regret}. We will denote by \(\F_{t-1}\) the filter \((x_1, y_1, \dots, x_t)\). 
We start by introducing the following comparison inequality. It is an equivalent to Lemma \ref{theorem:online CI} for the expected regret. It allows to control the expected regret with respect to \( \E[ \varphi_t | x_t ] \).

\begin{restatable}[Comparison Inequality in Expectation]{lmm}{}
    \label{theorem:CI expectancy}
    For any sequence of measurable functions \( (\hat g_t : \X \to \calH)_t\). For all \( t \in \llbracket T \rrbracket \), let \( f_t : \X\to \Z\) be defined by \( f_t(x) = \arg\min_{z\in\Z} \langle \psi(z), \hat g_t(x) \rangle_\calH \). Then we have
    \begin{equation}
        \E [ R_T ] \le 2 \cD \sqrt{T} \sqrt{ \sum_{t=1}^T \E [ \lVert \hat g_t(x_t) - \E [\varphi_t|x_t] \rVert^2_\calH ] } \,.
    \end{equation}
\end{restatable}

\begin{proof}
    We follow the proof from Lemma \ref{theorem:online CI} and get
    \[
        \E [ \Delta(f_t(x_t), y_t) - \Delta(f^*_t(x_t), y_t) ] \le \E [ \langle \psi( f_t(x_t)) - \psi(f^*_t(x_t)), \hat g_t(x_t) - \varphi_t \rangle ] \,.
    \]
    Remember that \(\varphi_t := \varphi(y_t) \) depends on \(y_t\). Now note that
    \begin{multline*}
        \E [ \langle \psi( f_t(x_t)) - \psi(f^*_t(x_t)), \E [\varphi_t | x_t] - \varphi_t \rangle ]
         = \langle \psi( f_t(x_t)) - \psi(f^*_t(x_t)), \E [ \E [\varphi_t | x_t] - \varphi_t ] \rangle\\
         = \langle \psi( f_t(x_t)) - \psi(f^*_t(x_t)), \E [ \E_{y_t} [ \E [\varphi_t | x_t] - \varphi_t | \F_{t-1} ] ] \rangle
         = 0 
    \end{multline*}
    since we condition on \(x_t\) in the expectation.
    Thus
    \begin{equation*}
        \E [ \Delta(f_t(x_t), y_t) - \Delta(f^*_t(x_t), y_t) ] \le \E [ \langle \psi( f_t(x_t)) - \psi(f^*_t(x_t)), \hat g_t(x_t) - \E [\varphi_t|x_t] \rangle ] \,.
    \end{equation*}
    We now apply successively Cauchy-Schwartz inequality and Jensen's inequality as in the proof of Lemma \ref{theorem:online CI} and obtain
    \begin{equation}
        \E [ \Delta(f_t(x_t), y_t) - \Delta(f^*_t(x_t), y_t) ] \le 2 \cD \sqrt{T} \sqrt{ \sum_{t=1}^T \E [ \lVert \hat g_t(x_t) - \E [\varphi_t|x_t] \rVert^2_\calH ] } \,.
    \end{equation}
    It concludes the proof.
\end{proof}

We now recall and prove the expectation of the regret bound.

\ThmExpectancyRegretBound*

\begin{proof}
    We apply Lemma \ref{theorem:CI expectancy}, then add and subtract a term to get
    \begin{align*}
        \E [ &\Delta(f_t(x_t), y_t) - \Delta(f^*_t(x_t), y_t) ] \\
        &\le 2 \cD \sqrt{T} \sqrt{ \sum_{t=1}^T \E [ \lVert \hat g_t(x_t) - \E [\varphi_t|x_t] \rVert^2_\calH ] } \\
        &= 2 \cD \sqrt{T} \sqrt{ \sum_{t=1}^T \E [ \lVert \hat g_t(x_t) - \E [\varphi_t|x_t] \rVert^2_\calH  - \lVert g^*(x_t) - \E [\varphi_t|x_t] \rVert^2_\calH + \lVert g^*(x_t) - \E [\varphi_t|x_t] \rVert^2_\calH]}\,.
    \end{align*}
    We bound the difference between the first two terms.
    \begin{align*}
        \sum_{t=1}^T &\E [ \lVert \hat g_t(x_t) - \E [\varphi_t|x_t] \rVert^2_\calH - \lVert g^*(x_t) - \E [\varphi_t|x_t] \rVert^2_\calH ] \\
        &= \sum_{t=1}^T \E [ \E_{y_t} [ \lVert \hat g_t(x_t) - \E [\varphi_t|x_t] \rVert^2_\calH - \lVert g^*(x_t) - \E [\varphi_t|x_t] \rVert^2_\calH | \F_{t-1} ] ] 
    \end{align*}
    Now note that 
    \begin{align*}
        \E_{y_t} [ \lVert \hat g_t(x_t) - \varphi_t \rVert^2 | \F_{t-1} ] = \E_{y_t} [ &\lVert \hat g_t(x_t) - \E [\varphi_t|x_t] \rVert^2 +\lVert \varphi_t - \E [\varphi_t|x_t] \rVert^2 | \F_{t-1} ] \\
        -2 \E_{y_t} [ &\langle \E [\varphi_t|x_t] - \varphi_t, \E [\varphi_t|x_t] - \hat g_t(x_t) \rangle | \F_{t-1} ] \\
        = \E_{y_t} [ &\lVert \hat g_t(x_t) - \E [\varphi_t|x_t] \rVert^2 +\lVert \varphi_t - \E [\varphi_t|x_t] \rVert^2 | \F_{t-1} ] 
    \end{align*}
    since we condition on \(x_t\).
    The same equality holds for \(g^*\). Therefore we obtain
    \begin{align*}
        \sum_{t=1}^T &\E [ \lVert \hat g_t(x_t) - \E [\varphi_t|x_t] \rVert^2_\calH - \lVert g^*(x_t) - \E [\varphi_t|x_t] \rVert^2_\calH ] \\
        &= \sum_{t=1}^T \E [ \lVert \hat g_t(x_t) - \varphi_t \rVert^2_\calH - \lVert g^*(x_t) - \varphi_t \rVert^2_\calH ] \\
        &= \E \left[ \sum_{t=1}^T  \lVert \hat g_t(x_t) - \varphi_t \rVert^2_\calH - \lVert g^*(x_t) - \varphi_t \rVert^2_\calH \right] \\
        &\le \E \left[ \deff(\lambda) \log\left(e + \tfrac{e\kappa^2T}{\lambda} \right) + \lambda \lVert g^* \rVert^2_\G \right] \\
        &= \deff(\lambda) \log\left(e + \tfrac{e\kappa^2T}{\lambda} \right) + \lambda \lVert g^* \rVert^2_\G
    \end{align*}
    where the inequality comes from Theorem \ref{theorem:regret vawk}. It concludes the proof.
\end{proof}

\ThmExpectancyRegretBoundApplied*

\begin{proof}
    We bound the effective dimension by \(\deff(\lambda) \le \tfrac{\kappa^2 T}{\lambda}\) and obtain
    \begin{equation*}
        \E \left[ R_T \right] \le 2 \cD \sqrt{T} \sqrt{\tfrac{\kappa^2 T}{\lambda} \log\left( e + \tfrac{e\kappa^2 T}{\lambda} \right) + \lambda \lVert g^* \rVert^2_\G} \,.
    \end{equation*}
    By choosing \( \lambda = \sqrt{T} \), we get
    \begin{equation*}
        \E \left[ R_T \right] \le 2 \cD T^{3/4} \sqrt{\kappa^2 \log\left(e+e\kappa^2\sqrt{T}\right) + \lVert g^* \rVert^2_\G} \,.
    \end{equation*}
\end{proof}

\section{Stochastic Regret Bounds in High Probability}\label{section:appendix stochastic regret bounds}
In this section we aim to retrieve the results from the supervised learning framework \citep{ILE} in the online learning setting.
For all time step \(t\), we define the filter \(\F_{t-1} = (x_1, y_1, \dots, x_{t-1}, y_{t-1}, x_t)\), and we denote by \( \E_t [.] \) the expectation \( \E_{y_t} [. | \F_{t-1} ] \).
We are now interested in bounding the cumulative risk \citep{wintenberger2024stochastic}
\begin{equation}
    \sum_{t=1}^T \E_t [\Delta(f_t (x_t), y_t) - \Delta(f^*_t(x_t), y_t) ] \,.
\end{equation}
Taking the expectation in \(y_t\) will avoid considering the noise of the random variables \((y_t)_t\) and allow us to obtain closer results from \cite{ILE} by bounding with respect to \(\E[\varphi_t|x_t]\) instead of~\(\varphi_t\).

We use the proof of Theorem~1 of \cite{vanderhoeven2023highprobability}, which allows to bound the cumulative risk in the feature space with high probability using a regret bound for an exp-concave loss. Moreover, \cite{vanderhoeven2023highprobability} enables us to aggregate our predictors into a unique function \(\bar f_T\) and bound its cumulative risk in high probability, which is a setting similar to the supervised learning study.
In order to apply this theorem, we need to modify our feature predictor \( \hat g_t : \X \to \calH \) and define it using a shifted version of the losses
\begin{equation}\label{eq:hat gt shifted}
    \hat g_t = \arg \min_{g\in\G} \sum_{s=1}^{t-1} \lVert \tfrac{1}{2} g(x_s) + \tfrac{1}{2} \hat g_s(x_s) - \varphi_s \rVert^2_\calH + \lambda \lVert g \rVert^2_\G + \tfrac{1}{4} \lVert g(x_t) \rVert^2_\calH \,.
\end{equation}
The predictor \( f_t : \X \to \Z \) is then defined as follows
\begin{equation}\label{eq:ft high proba}
    f_t(x) = \arg\min_{z\in\Z} \langle \psi(z), \hat g_t(x) \rangle_\calH = \arg\min_{z\in\Z} \sum_{s=1}^{t-1} \beta_s(x) \Delta(z, y_s) \,.
\end{equation}
As previously, we do not require the knowledge of \(\psi, \varphi\) and \(\calH\) to make a prediction.

We first introduce a comparison inequality. It is an equivalent to Lemma \ref{theorem:online CI} for the cumulative risk.

\begin{restatable}[Comparison Inequality for Cumulative Risk]{lmm}{}
\label{theorem:CI cumulative risk}
    For any sequence of measurable functions \( (\hat g_t : \X \to \calH)_t\). For all \( t \in \llbracket T \rrbracket \), let \( f_t : \X\to \Z\) be defined by \( f_t(x) = \arg\min_{z\in\Z} \langle \psi(z), \hat g_t(x) \rangle_\calH \). Then we have
    \begin{equation*}
        \sum_{t=1}^T \E_t [\Delta( f_t (x_t), y_t) - \Delta(f^*_t(x_t), y_t) ]
        \le 2 \cD \sqrt{T} \sqrt{ \sum_{t=1}^T \E_t [ \lVert \hat g_t(x_t) - \E [\varphi_t | x_t] \rVert^2_\calH ] } \,.
    \end{equation*}
\end{restatable}

\begin{proof}
    We recall the notation \( \varphi_t = \varphi(y_t) \), which therefore depends on \(y_t\) in the expectation.
    We start by adding and subtracting two terms
    \begin{align*}
        \E_t [\Delta( f_t (x_t), y_t) &- \Delta(f^*_t(x_t), y_t) ] \\
        = \E_t [ &\langle \psi ( f_t(x_t)), \varphi_t \rangle - \langle \psi (f^*_t(x_t)), \varphi_t \rangle ] \\
        = \E_t [ &\langle \psi ( f_t(x_t)), \varphi_t \rangle - \langle \psi ( f_t(x_t)), \hat g_t(x_t) \rangle \\
        + &\langle \psi ( f_t(x_t)), \hat g_t(x_t) \rangle - \langle \psi (f^*_t(x_t)), \hat g_t(x_t) \rangle \\
        + &\langle \psi (f^*_t(x_t)), \hat g_t(x_t) \rangle - \langle \psi (f^*_t(x_t)), \varphi_t \rangle ] \\
        \le \E_t [ &\langle \psi ( f_t(x_t)), \varphi_t \rangle - \langle \psi ( f_t(x_t)), \hat g_t(x_t) \rangle \\
        + &\langle \psi (f^*_t(x_t)), \hat g_t(x_t) \rangle - \langle \psi (f^*_t(x_t)), \varphi_t \rangle ] \\
        = \E_t [ &\langle \psi( f_t(x_t)) - \psi(f^*_t(x_t)), \hat g_t(x_t) - \varphi_t \rangle ]
    \end{align*}
    
    where the inequality comes by definition of \( f_t\). Now note that
    \begin{equation*}
        \E_t [ \langle \psi( f_t(x_t)) - \psi(f^*_t(x_t)), \E [\varphi_t | x_t] - \varphi_t \rangle ] 
        = \langle \psi( f_t(x_t)) - \psi(f^*_t(x_t)), \E_t [ \E [\varphi_t | x_t] - \varphi_t ] \rangle = 0 
    \end{equation*}
    since we condition on \(x_t\) in \(\E_t\).
    Thus
    \begin{align*}
        \E_t [ \langle \psi( f_t(x_t)) &- \psi(f^*_t(x_t)), \hat g_t(x_t) - \varphi_t \rangle ] \\
        = \E_t [ &\langle \psi( f_t(x_t)) - \psi(f^*_t(x_t)), \hat g_t(x_t) - \E [\varphi_t | x_t] \rangle ] \\
        + \E_t [ &\langle \psi( f_t(x_t)) - \psi(f^*_t(x_t)), \E [\varphi_t | x_t] - \varphi_t \rangle ] \\
        = \E_t [ &\langle \psi( f_t(x_t)) - \psi(f^*_t(x_t)), \hat g_t(x_t) - \E [\varphi_t | x_t] \rangle ] \,.
    \end{align*}
    We now apply successively Cauchy-Schwartz inequality and Jensen's inequality.
    \begin{align*}
        \sum_{t=1}^T \E_t [ \langle \psi( f_t(x_t)) - \psi(f^*_t(x_t)), \hat g_t(x_t) - \E [\varphi_t | x_t] \rangle ] &\le 2 \cD \sum_{t=1}^T \E_t [ \lVert \hat g_t(x_t) - \E [\varphi_t | x_t] \rVert_\calH ] \\
        &\le 2 \cD \sqrt{T} \sqrt{ \sum_{t=1}^T \E_t [ \lVert \hat g_t(x_t) - \E [\varphi_t | x_t] \rVert^2_\calH ] }
    \end{align*}
\end{proof}

Let \(\eta = 1/2(\kappa \sup \lVert g \rVert_\G + 1)^2 \) be an exp-concavity constant of the loss \( \ell_t(g) = \lVert g(x_t) - \varphi_t \rVert^2_\calH \). Let \(m=2(\kappa \sup \lVert g \rVert_\G + 1)^2\) be such that \( \ell(g) - \ell(g') \le m \) for all \( g,g' \in \G \). We define 
\begin{equation}\label{eq:gamma}
    \gamma = 4 \max(\tfrac{1}{\eta}, m)
\end{equation}
as in the Theorem 1 of \cite{vanderhoeven2023highprobability}.

We bound the cumulative risk of our algorithm, see Theorem \ref{theorem:Avg cumulative risk}.
Compared to our previous result Theorem \ref{theorem:FirstRegretBound}, we obtain a bound in high probability. We now use \(g^*\in\G\) to model \((\E[\varphi_t|x_t])_t\) instead of \( (\varphi_t)_t \). This difference allows us not to consider the noise of the random variables \( x_t \mapsto \varphi_t \). It also allows us to be closer to the framework of \cite{ILE} that compares a model \(g\in\G\) with the optimum and conditional expectation \( x \mapsto \int_\Y \varphi(y) d\rho(y|x) \).

\begin{restatable}[Average Cumulative Risk]{thm}{ThmAverageCumulativeRisk}
    \label{theorem:Avg cumulative risk}
    Let \( (f_t)_t \) be defined as in \eqref{eq:ft high proba}. Let \( \delta \in (0,1] \) and \( \gamma = 8 (\kappa \sup \lVert g \rVert_\G + 1)^2 \). With \( \lambda=\sqrt{T} \) and assuming that there exists a function \(h\in\G\) such that for all \(t\in\llbracket T\rrbracket, h(x_t) = \E[\varphi_t|x_t] \) , we have with probability \( 1-\delta \)
    \begin{align*}
        \dfrac{1}{T} \sum_{t=1}^T \E_t [\Delta(&f_t (x_t), y_t) - \Delta(f^*_t(x_t), y_t) ] \\
        &\le 2\cD T^{-1/4} \sqrt{ \tfrac{B^2\kappa^2}{8} \log\left(e + \tfrac{e\kappa^2 \sqrt{T}}{4} \right)
        + 2 \lVert g^* \rVert^2 } + 2\cD T^{-1/2} \sqrt{2\gamma\log(\delta^{-1})} \\
        &= O\left(T^{-1/4} \sqrt{\log (T)} + T^{-1/2} \sqrt{\log (\delta^{-1})} \right) .
    \end{align*}
\end{restatable}

The assumption we do on \(h\) is the same that the one that is done on \(g^*\) to obtain the convergence rate in \cite{ILE}, and is a common assumption for Kernel Ridge Regression \citep{caponnetto2007optimal, steinwart2008support}. Up to the log factor, we retrieve the same bound in the online learning setting and without assuming that the data are i.i.d. Therefore our result is more general than the original result in the batch statistical framework, however we are using \(T\) different functions to predict the outputs. In Theorem \ref{theorem:excess risk aggregate}, we provide a similar result with a single estimator obtained by aggregation.

\begin{proof}
    \textbf{Step 1: Controlling the regret of \((f_t)\) by the regret of \((\hat g_t)\).}
    Applying Lemma \ref{theorem:CI cumulative risk}, we obtain
    \begin{align*}
        &\sum_{t=1}^T \E_t [\Delta( f_t (x_t), y_t) - \Delta(f^*_t(x_t), y_t)] \\
        &\le 2 \cD \sqrt{T} \sqrt{ \sum_{t=1}^T \E_t [ \lVert \hat g_t(x_t) - \E [\varphi_t | x_t] \rVert^2_\calH ] } \\
        &= 2 \cD \sqrt{T} \sqrt{ \sum_{t=1}^T \E_t [ \lVert \hat g_t(x_t) - \E [\varphi_t | x_t] \rVert^2_\calH - \lVert g^*(x_t) - \E [\varphi_t | x_t] \rVert^2_\calH + \lVert g^*(x_t) - \E [\varphi_t | x_t] \rVert^2_\calH ] } \,.
    \end{align*}
    
    Now note that 
    \begin{align*}
        \E_t [ \lVert \hat g_t(x_t) - \varphi_t \rVert^2 ] = \E_t [ &\lVert \hat g_t(x_t) - \E [\varphi_t|x_t] \rVert^2 +\lVert \varphi_t - \E [\varphi_t|x_t] \rVert^2 ] \\
        -2 \E_t [ &\langle \E [\varphi_t|x_t] - \varphi_t, \E [\varphi_t|x_t] - \hat g_t(x_t) \rangle ] \\
        = \E_t [ &\lVert \hat g_t(x_t) - \E [\varphi_t|x_t] \rVert^2 +\lVert \varphi_t - \E [\varphi_t|x_t] \rVert^2 ] 
    \end{align*}
    since we condition on \(x_t\) in \(\E_t\).
    The same equality holds for $g^*(x_t)$. Thus 
    \begin{align*}
        \sum_{t=1}^T &\E_t [ \lVert \hat g_t(x_t) - \E[\varphi_t|x_t] \rVert^2 - \lVert g^*(x_t) - \E[\varphi_t|x_t] \rVert^2  ] \\
        &= \sum_{t=1}^T \E_t [ \lVert \hat g_t(x_t) - \varphi_t \rVert^2 - \lVert g^*(x_t) - \varphi_t \rVert^2 ] \,.
    \end{align*}

    \textbf{Step 2: Bounding the regret of the shifted \textit{KAAR} estimator.}
    We note that 
    \begin{align*}
        \hat g_t &= \arg \min_{g\in\G} \sum_{s=1}^{t-1} \lVert \tfrac{1}{2} g(x_s) + \tfrac{1}{2} \hat g_s(x_s) - \varphi_s \rVert^2_\calH + \lambda \lVert g \rVert^2_\G + \tfrac{1}{4} \lVert g(x_t) \rVert^2_\calH \\
        &= \arg \min_{g\in\G} \tfrac{1}{4} \sum_{s=1}^{t-1} \lVert g(x_s) + \hat g_s(x_s) - 2 \varphi_s \rVert^2_\calH + \lambda \lVert g \rVert^2_\G + \tfrac{1}{4} \lVert g(x_t) \rVert^2_\calH \\
        &= \arg \min_{g\in\G} \sum_{s=1}^{t-1} \lVert g(x_s) + \hat g_s(x_s) - 2 \varphi_s \rVert^2_\calH + 4 \lambda \lVert g \rVert^2_\G + \lVert g(x_t) \rVert^2_\calH \,.
    \end{align*}
    Thus the function \(\hat g_t\) aims to estimate \( 2 \varphi_s - \hat g_s(x_s)\) with a regularisation parameter \(4\lambda\). We apply Lemma \ref{theorem:regret vawk} and bound \( 2 \varphi_s - \hat g_s(x_s)\) for all \(s\in\llbracket T \rrbracket\),
    \begin{equation*}
        \lVert 2 \varphi_s - \hat g_s(x_s) \rVert_\calH \le 2 \lVert \varphi_s \rVert + \lVert \hat g_s(x_s) \rVert \le 2 + \lVert \phi(x_s) \rVert \sup_{g\in\G} \lVert g \rVert_\G \le 2 + \kappa \sup_{g\in\G} \lVert g \rVert_\G =: B \,.
    \end{equation*}
    We get
    \begin{equation*}
        \sum_{t=1}^T \lVert 2 \hat g_t(x_t) - 2 \varphi_t \rVert^2_\calH \le B^2 \log \left( e + \tfrac{e\kappa^2 T}{4\lambda} \right) \deff(4\lambda) + \min_{g\in\G} \left[ \sum_{t=1}^T \lVert g(x_t) + \hat g_t(x_t) - 2 \varphi_t \rVert^2_\calH + 4 \lambda \lVert g \rVert^2_\G \right] \,.
    \end{equation*}
    We divide by 4 on both sides and obtain
    \begin{equation}
        \sum_{t=1}^T \lVert \hat g_t(x_t) - \varphi_t \rVert^2_\calH \le \tfrac{B^2}{4} \log \left( e + \tfrac{e\kappa^2 T}{4\lambda} \right) \deff(4\lambda) + \min_{g\in\G} \tilde L_T (g) \,.
    \end{equation}

    \textbf{Step 3: Applying \cite{vanderhoeven2023highprobability} to bound in high probability.}
    We apply Theorem 1 of \cite{vanderhoeven2023highprobability} to obtain with probability \(1-\delta\),

    \begin{align}
        &\sum_{t=1}^T \E_t [\Delta( f_t (x_t), y_t) - \Delta(f^*_t(x_t), y_t)] \nonumber \\
        &\le 2 \cD \sqrt{T} \left( \tfrac{B^2}{2} \deff(4\lambda)\log\left(e + \tfrac{e\kappa^2 T}{4\lambda} \right) + 2\gamma\log(\delta^{-1})
        + 2 \lambda \lVert g^* \rVert^2 + \sum_{t=1}^T \lVert g^*(x_t) - \E[\varphi_t|x_t] \rVert^2 \right)^{1/2}
    \end{align}
    where \(\gamma\) is defined in Eq.~\eqref{eq:gamma} and \( B = 2 + \kappa \sup_{g\in\G} \lVert g \rVert_\G\).
    We bound the effective dimension by \( \deff(\lambda) \le \tfrac{\kappa^2 T}{\lambda} \) and obtain
    \begin{equation*}
        \sum_{t=1}^T \E_t [\Delta(f_t (x_t), y_t) - \Delta(f^*_t(x_t), y_t) ] \le 2\cD \sqrt{T} \left( \tfrac{B^2\kappa^2T}{8\lambda} \log\left(e + \tfrac{e\kappa^2 T}{4\lambda} \right) + 2\gamma\log(\delta^{-1})
        + 2 \lambda \lVert g^* \rVert^2 \right)^{1/2} \,.
    \end{equation*}
    By choosing \( \lambda=\sqrt{T} \) we get
    \begin{multline*}
        \sum_{t=1}^T \E_t [\Delta(f_t (x_t), y_t) - \Delta(f^*_t(x_t), y_t) ] \\
        \le 2\cD \sqrt{T} \left( \tfrac{B^2\kappa^2 \sqrt{T}}{8} \log\left(e + \tfrac{e\kappa^2 \sqrt{T}}{4} \right) + 2\gamma\log(\delta^{-1})
        + 2 \sqrt{T} \lVert g^* \rVert^2 \right)^{1/2} \,.
    \end{multline*}
\end{proof}

\subsection{Aggregating into a Unique Predictor}\label{section:appendix aggregation}

In this section, we build a unique predictor and we consider that the data \( (x_t, y_t)_t \) are i.i.d., in order to be as close as possible from the supervised learning framework. As the output space \(\Z\) does not have a vectorial structure, we cannot aggregate the \((f_t)_t\). Indeed, let \(f, f' : \X \to \Z\), there is no guarantee that \(f+ f'\) takes values in \(\Z\) as well.
Therefore we build a unique predictor \(\bar f_T \) from the \( T \) already computed feature predictors \( (\hat g_t)_t \). We build an aggregate \( \bar g_T : \X \to \calH \) defined as the average of the~\(\hat g_t\),
\begin{equation}
    \bar g_T = \tfrac{1}{T} \sum_{t=1}^T \hat g_t 
\end{equation}
and define the predictor \( \bar f_T : \X \to \Z \) with respect to \( \bar g_T \) as follows
\begin{equation}\label{eq:bar fT}
    \bar f_T(x) = \arg\min_{z\in\Z} \langle \psi(z), \bar g_T(x) \rangle_\calH \,.
\end{equation}
We bound the excess risk of the predictor \(\bar f_T\).

\begin{restatable}[Excess Risk]{thm}{ThmExcessRiskAggregate}
\label{theorem:excess risk aggregate}
    Let \( f^* : \X \to \Z \) be a measurable function and \( \gamma = 8 (\kappa \sup \lVert g \rVert_\G + 1)^2 \), with \(\lambda = \sqrt{T}\) and assuming there is a function \(g^*\in\G\) such that \( \E_x \E_y [ \lVert g^*(x) - \E [\varphi(y)|x] \rVert^2 | x ] = 0 \). Let \(\bar f_T\) be defined as in \eqref{eq:bar fT} and let \( \delta \in (0,1] \). With probability \(1-\delta\), we have
    \begin{align*}
        \E_{x,y} [\Delta(&\bar f_T(x), y) - \Delta(f^*(x), y)] \\
        &\le 2\cD T^{-1/4} \sqrt{ \tfrac{B^2\kappa^2}{8} \log\left(e + \tfrac{e\kappa^2 \sqrt{T}}{4} \right)
        + 2 \lVert g^* \rVert^2 } + 2\cD T^{-1/2} \sqrt{2\gamma\log(\delta^{-1})} \\
        &= O\left(T^{-1/4} \sqrt{\log (T)} + T^{-1/2} \sqrt{\log (\delta^{-1})} \right) \,.
    \end{align*}
\end{restatable}

\begin{proof}
    We follow the proof of Lemma \ref{theorem:CI cumulative risk}, with the difference that at the end we apply Jensen's inequality with respect to the expectation to obtain a comparison inequality
    \begin{align*}
        \E_{x,y} &[\Delta(\bar f_T(x), y) - \Delta(f^*(x), y)] \\
        &\le \E_{x,y} [ \langle \psi(\bar f_T(x)) - \psi(f^*(x)), \bar g_T(x) - \varphi(y) \rangle ] \\
        &= \E_x \E_y [ \langle \psi(\bar f_T(x)) - \psi(f^*(x)), \bar g_T(x) - \E [\varphi(y)|x] \rangle | x ] \\
        &+ \E_x \E_y [ \langle \psi(\bar f_T(x)) - \psi(f^*(x)), \E [\varphi(y)|x] - \varphi(y) \rangle | x ] \\
        &= \E_x \E_y [ \langle \psi(\bar f_T(x)) - \psi(f^*(x)), \bar g_T(x) - \E [\varphi(y)|x] \rangle | x ] \\
        &\le 2 \cD \E_x \E_y [ \lVert \bar g_T(x) - \E [\varphi(y)|x] \rVert | x ] \\
        &\le 2 \cD \sqrt{ \E_x \E_y [ \lVert \bar g_T(x) - \E [\varphi(y)|x] \rVert^2 | x ]} \,.
    \end{align*}
    We now apply Theorem 1 from \cite{vanderhoeven2023highprobability} using Step 2 of the proof of Theorem \ref{theorem:Avg cumulative risk}.
    \begin{align*}
        \E_{x,y} &[\Delta(\bar f_T(x), y) - \Delta(f^*(x), y)] \\
        &= 2 \cD \left( \E_x \E_y [ \lVert \bar g_T(x) - \E [\varphi(y)|x] \rVert^2 - \lVert g^*(x) - \E [\varphi(y)|x] \rVert^2 + \lVert g^*(x) - \E [\varphi(y)|x] \rVert^2 | x ] \right)^{1/2} \\
        &= 2 \cD \left( \E_x \E_y [ \lVert \bar g_T(x) - \varphi(y) \rVert^2 - \lVert g^*(x) - \varphi(y) \rVert^2 + \lVert g^*(x) - \E [\varphi(y)|x] \rVert^2 | x ] \right)^{1/2} \\
        &= 2 \cD \left( \E_{x,y} [ \lVert \bar g_T(x) - \varphi(y) \rVert^2 - \lVert g^*(x) - \varphi(y) \rVert^2 ] + \E_x \E_y [ \lVert g^*(x) - \E [\varphi(y)|x] \rVert^2 | x ] \right)^{1/2} \\
        &\le 2\cD \left( \tfrac{ \tfrac{B^2}{2} \deff(4\lambda) \log\left(e+\tfrac{e\kappa^2 T}{4\lambda}\right) + 2\lambda\lVert g^* \rVert^2_\G + 2 \gamma \log(\delta^{-1}) }{T} + \E_x \E_y [ \lVert g^*(x) - \E [\varphi(y)|x] \rVert^2 | x ] \right)^{1/2} \\
        &= 2\cD \left( \dfrac{ \tfrac{B^2}{2} \deff(4\lambda) \log\left(e+\tfrac{e\kappa^2 T}{4\lambda}\right) + 2\lambda\lVert g^* \rVert^2_\G + 2 \gamma \log(\delta^{-1}) }{T} \right)^{1/2}
    \end{align*}
    where \( B= 2 + \kappa \sup_{g\in\G} \lVert g \rVert_\G \).
    Upper-bounding the effective dimension by \( \deff(4\lambda) \le \tfrac{\kappa^2 T}{4\lambda} \) and choosing \( \lambda = \sqrt{T} \) gives the desired result.
\end{proof}

\section{Proof of Theorem \ref{theorem:expected regret non stationary environment}: Dealing with Non-Stationary Data in Expectation}\label{section:appendix non stationary}

We define \(m\in\mathbb N\) such that \(1=t_1 \le t_2 \le \dots \le t_{m+1} = T+1\) and such that 
\begin{equation}\label{eq:definition m}
    \sum_{t=t_i +1}^{t_{i+1} - 1} \lVert g_t^* - g^*_{t-1} \rVert_\G \le \dfrac{V_\G}{m} \quad \text{for all } i\in\llbracket 1, m \rrbracket .
\end{equation}
That is to say that the variation of \((g_t^*)_t\) is small between \(t_i\) and \(t_{i+1}-1\). Note that the sum \( \smash{\sum_{i=1}^m \sum_{t=t_i +1}^{t_{i+1} - 1} \lVert g_t^* - g^*_{t-1} \rVert_\G} \) does not take into account the norms of \( \lVert g^*_{t_i} - g^*_{t_i-1} \rVert_\G \) for \( i\in\llbracket 2, m\rrbracket \).
As in \cite{raj2020non}, we define an approximation \( \smash{(g^*_{t_i:t_{i+1}})_{i=1}^m \in \G^m} \) of \( (g_t^*)_t \) with only \(m\) changes through the \(T\) time steps that occur between \(t_i\) and \(t_{i+1}\). It is an hypothetical forecaster with \(m\) restart times. Formally we define
\begin{equation}\label{eq:discrete approx of ht}
    \bar g_{t_i:t_{i+1}} := \arg\min_{g\in\G} \sum_{t=t_i}^{t_{i+1}-1} \lVert g - g_t^* \rVert_\G^2 = \dfrac{1}{t_{i+1} - t_i} \sum_{t=t_i}^{t_{i+1}-1} g^*_t
\end{equation}
and by \(\bar g_t\) we denote \(\bar g_{t_i:t_{i+1}}\) for all \(t\in\llbracket t_i, t_{i+1}-1 \rrbracket\).

We bound the dynamic regret of the \textit{KAAR} estimator, see Proposition \ref{theorem:dynamic regret vawk}.
It is expressed with respect to the time dependent effective dimension \( \deff(\lambda, s-r) \) defined as
\begin{equation}
    \deff(\lambda, s-r) := Tr( K_{s-r, s-r} (K_{s-r, s-r} + \lambda I )^{-1}) \quad \forall \lambda>0,
\end{equation}
where \( K_{s-r, s-r} \in \R^{(s-r-1)\times(s-r-1)} \) is defined by \( (K_{s-r, s-r})_{ij} = k(x_{r+i-1}, x_{r+j-1}) \).

\begin{restatable}[Dynamic Regret of \textit{KAAR}]{prop}{}
\label{theorem:dynamic regret vawk}
    Let \( (\hat g_t)_t \) be defined as in \eqref{eq:hat gt dynamic} and let \( (g^*_t)_t \in \G^T \). Let \(m \in \mathbb N\) be defined as in \eqref{eq:definition m}. Let \( \eta = 1/2(\kappa \sup \lVert g \rVert_\G + 1)^2 \). Then we have
    \begin{multline*}
        \sum_{t=1}^T \lVert \hat g_t(x_t) - \varphi_t \rVert^2_\calH - \lVert g^*_t(x_t) - \varphi_t \rVert^2_\calH \\
        \le \tfrac{m \log T}{\eta} + \log \left( e + \tfrac{e \kappa^2 T}{\lambda} \right) \sum_{i=1}^m \deff(\lambda, t_{i+1}-t_i) + \lambda m \max_{t\in\llbracket T \rrbracket} \lVert g^*_t \rVert^2_\G + \tfrac{ 4 \kappa V_\G T}{m}
        =: R_T(\lambda, m) \,.
    \end{multline*}
\end{restatable}

\begin{proof}
    \textbf{Step 1: We add two intermediary terms.}
    We introduce two new terms in the sum and bound the differences separately.
    \begin{align}\label{eq:dynamic regret kaar sum}
        \sum_{t=1}^T \lVert \hat g_t(x_t) - \varphi_t \rVert^2_\calH - \lVert g^*_t(x_t) - &\varphi_t \rVert^2_\calH 
         = \sum_{i=1}^m \sum_{t=t_i}^{t_{i+1}-1} \lVert \hat g_t(x_t) - \varphi_t \rVert^2_\calH - \lVert g^*_t(x_t) - \varphi_t \rVert^2_\calH \nonumber \\
        = \sum_{i=1}^m \sum_{t=t_i}^{t_{i+1}-1} &\lVert \hat g_t(x_t) - \varphi_t \rVert^2_\calH -\lVert \hat g_{t_i:t}(x_t) - \varphi_t \rVert^2_\calH \\
        + &\lVert \hat g_{t_i:t}(x_t) - \varphi_t \rVert^2_\calH - \lVert \bar g_{t_i:t_{i+1}}(x_t) - \varphi_t \rVert^2_\calH \nonumber \\
        + &\lVert \bar g_{t_i:t_{i+1}}(x_t) - \varphi_t \rVert^2_\calH - \lVert g^*_t(x_t) - \varphi_t \rVert^2_\calH \nonumber
    \end{align}

    \textbf{Step 2: Bounding the first difference.} For all \(k^*\in\llbracket T \rrbracket\) and all \(s_2 \in \llbracket T \rrbracket\), we prove that
    \begin{equation}
        \sum_{t=1}^{s_2} \left( \ell_t\left( \sum_{k=1}^K p_t(k) \hat g_{k:t} \right) - \ell_t(\hat g_{k^*:t}) \right)\mathbbm{1}[k^*\le t] \le \dfrac{\log K}{\eta} \,.
    \end{equation}
    The proof is based on the proof of EWA applied to \(\tilde \ell_t\). Let \(W_t\) be the normalisation constant of \(w_t\).
    \begin{align*}
        W_{s_2 +1} &= \sum_{k=1}^K \exp \left( -\eta \sum_{s=1}^{s_2} \tilde \ell_s(k) \right) \\
        &= \sum_{k=1}^K \exp \left( -\eta \sum_{s=1}^{s_2-1} \tilde \ell_s(k) \right) \exp(-\eta \tilde \ell_{s_2} (k) ) \\
        &= W_{s_2} \sum_{k=1}^K w_{s_2}(k) \exp(-\eta \tilde \ell_{s_2} (k) ) \\
        &\le W_{s_2} \exp \left( -\eta \ell_{s_2} \left( \sum_{k=1}^K w_{s_2}(k) \tilde g_{k:s_2} \right) \right)
    \end{align*}
    where the inequality comes from Jensen's inequality and $\eta$-exp-concavity of \(\ell_{s_2}\), and where we define
    \begin{equation*}
        \tilde g_{k:t} =
        \begin{cases}
            \hat g_{k:t} \text{ if } k\le t \\
            \hat g_t \text{ if } k > t
        \end{cases} .
    \end{equation*}
    Now note that
    \begin{align*}
        \sum_{k=1}^K w_{s_2}(k) \tilde g_{k:s_2} &= \sum_{k\le s_2} w_{s_2}(k) \hat g_{k:s_2} + \sum_{k> s_2} w_{s_2}(k) \hat g_{s_2} \\
        &= \left( \sum_{k\le s_2} w_{s_2}(k) \right) \sum_{k\le s_2} p_{s_2}(k) \hat g_{k:s_2} + \sum_{k > s_2} w_{s_2}(k) \hat g_{s_2} \\
        &= \hat g_{s_2} \left( \sum_{k\le s_2} w_{s_2}(k) + \sum_{k > s_2} w_{s_2}(k) \right) \\
        &= \hat g_{s_2} .
    \end{align*}
    Thus by induction we obtain
    \begin{equation*}
        W_{s_2 +1} \le W_1 \exp \left( -\eta \sum_{s=1}^{s_2} \ell_s(\hat g_s) \right) \le K \exp \left( -\eta \sum_{s=1}^{s_2} \ell_s(\hat g_s) \right)
    \end{equation*}
    where the right inequality is obtained by choosing to initialise \(w_1\) as the uniform probability over the \(K\) experts. We now compute a lower bound of \(W_{s_2 +1}\),
    \begin{align*}
        W_{s_2 +1} &= \sum_{k=1}^K \exp \left( -\eta \sum_{s=1}^{s_2} \tilde \ell_s(k) \right) \\
        &\ge \exp \left( -\eta \sum_{s=1}^{s_2} \tilde \ell_s(k^*) \right) \\
        &= \exp \left( -\eta \sum_{s=1}^{s_2} \ell_s(\hat g_{k^*:s}) \mathbbm 1 [k^*\le s] + \ell_s(\hat g_s) \mathbbm 1 [k^* > s] \right) .
    \end{align*}
    By taking the log we obtain the desired result.
    
    \textbf{Step 3: Bounding the second difference.}
    We have that
    \begin{align*}
        \sum_{t=t_i}^{t_{i+1}-1} &\lVert \hat g_{t_i:t}(x_t) - \varphi_t \rVert^2_\calH - \lVert \bar g_{t_i:t_{i+1}}(x_t) - \varphi_t \rVert^2_\calH \\
        &\le \deff(\lambda, t_{i+1}-t_i) \log \left( e + \tfrac{e\kappa^2 (t_{i+1}-t_i)}{\lambda}\right) + \lambda \lVert \bar g_{t_i:t_{i+1}} \rVert^2_\G .
    \end{align*}
    This is a generalisation of Theorem \ref{theorem:regret vawk} with a late starting point.

    \textbf{Step 4: Bounding the third difference.}
    We bound the third term of the sum using the following derivation.
    \begin{align*}
        &\lVert \bar g_{t_i:t_{i+1}}(x_t) - \varphi_t \rVert^2_\calH - \lVert g^*_t(x_t) - \varphi_t \rVert^2_\calH \\
        &= - \left\lVert g^*_t(x_t) - \bar g_{t_i:t_{i+1}}(x_t) \right\rVert^2_\calH + 2 \left\langle \varphi_t - \bar g_{t_i:t_{i+1}}(x_t), g^*_t(x_t) - \bar g_{t_i:t_{i+1}}(x_t) \right\rangle_\calH \\
        &\le 2 \left\lVert \varphi_t - \bar g_{t_i:t_{i+1}}(x_t) \right\rVert_\calH \left\lVert g^*_t(x_t) - \bar g_{t_i:t_{i+1}}(x_t) \right\rVert_\calH \\
        &\le 4 \kappa \left\lVert g^*_t - \bar g_{t_i:t_{i+1}} \right\rVert_\G
    \end{align*}
    We bound the norm by
    \begin{align*}
        \lVert g^*_t - \bar g_{t_i:t_{i+1}} \rVert_\G &= \left\lVert g^*_t - \tfrac{1}{t_{i+1} - t_i} \sum_{s=t_i}^{t_{i+1}-1} \bar g_s \right\rVert_\G \\
        &= \left\lVert \tfrac{1}{t_{i+1} - t_i} \sum_{s=t_i}^{t_{i+1}-1} g^*_t - \bar g_s \right\rVert_\G \\
        &\le \tfrac{1}{t_{i+1} - t_i} \sum_{s=t_i}^{t_{i+1}-1} \lVert g^*_t - \bar g_s \rVert_\G \\
        &\le \max_{s\in\llbracket t_i, t_{i+1}-1 \rrbracket} \lVert g^*_t - \bar g_s \rVert_\G
    \end{align*}
    where the first inequality comes from Jensen's inequality. We now separate the max in two terms at time step \(t\), and use a telescopic sum.
    \begin{align*}
        \max_{s\in\llbracket t_i, t_{i+1}-1 \rrbracket} \lVert g^*_t - \bar g_s \rVert_\G &\le \max_{s\in\llbracket t_i, t-1 \rrbracket} \lVert g^*_t - \bar g_s \rVert_\G + \max_{s\in\llbracket t+1, t_{i+1}-1 \rrbracket} \lVert g^*_t - \bar g_s \rVert_\G \\
        &= \max_{s\in\llbracket t_i, t-1 \rrbracket} \left\lVert \sum_{r=s+1}^t g^*_r- \bar g_{r-1} \right\rVert + \max_{s\in\llbracket t+1, t_{i+1}-1 \rrbracket} \left\lVert \sum_{r=t+1}^s g^*_r- \bar g_{r-1} \right\rVert \\
        &\le \max_{s\in\llbracket t_i, t-1 \rrbracket} \sum_{r=s+1}^t \lVert g^*_r- \bar g_{r-1} \rVert + \max_{s\in\llbracket t+1, t_{i+1}-1 \rrbracket} \sum_{r=t+1}^s \lVert g^*_r- \bar g_{r-1} \rVert \\
        &= \sum_{r=t_i+1}^{t_{i+1}-1} \lVert g^*_r- \bar g_{r-1} \rVert \\
        &\le V_\G/m
    \end{align*}
    We may now sum to obtain
    \begin{align*}
        4\kappa\sum_{t=1}^T \lVert g^*_t - \bar g_t \rVert_\G &= 4\kappa\sum_{i=1}^m \sum_{t=t_i}^{t_{i+1}-1} \lVert g^*_t - \bar g_{t_i:t_{i+1}} \rVert_\G \\
        &\le 4\kappa\sum_{i=1}^m \sum_{t=t_i}^{t_{i+1}-1} V_\G/m \\
        &= 4\kappa V_\G/m \sum_{i=1}^m t_{i+1} - t_i \\
        &= 4\kappa V_\G T /m .
    \end{align*}

    \textbf{Step 5: Putting everything together.} We obtain
    \begin{align*}
        &\sum_{t=1}^T \lVert \hat g_t(x_t) - \varphi_t \rVert^2_\calH - \lVert g^*_t(x_t) - \varphi_t \rVert^2_\calH \\
        &\le \tfrac{m \log T}{\eta} + \log \left( e + \tfrac{e \kappa^2 T}{\lambda} \right) \sum_{i=1}^m \deff(\lambda, t_{i+1}-t_i) + \lambda \sum_{i=1}^m \lVert \bar g_{t_i:t_{i+1}} \rVert^2 + \tfrac{4 \kappa V_\G T}{m} .
    \end{align*}
    We conclude by noting that
    \begin{equation*}
        \left\lVert \bar g_{t_i:t_{i+1}} \right\rVert^2_\G = \left\lVert \tfrac{1}{t_{i+1}-t_i} \sum_{t=t_i}^{t_{i+1}-1} g^*_t \right\rVert^2_\G \le \tfrac{1}{t_{i+1}-t_i} \sum_{t=t_i}^{t_{i+1}-1} \lVert g^*_t \rVert^2_\G \le \max_{t\in\llbracket t_i, t_{i+1}-1 \rrbracket} \lVert g^*_t \rVert^2_\G
    \end{equation*}
    where the first inequality is by Jensen's inequality.
\end{proof}

We recall and prove our main result from Section \ref{section:dynamic regret}.
\ThmExpectedRegretNonStationaryEnvironment*

Precisely if \( \lambda = V_\G^{-1/3} T^{1/3} \), we have 

\begin{equation*}
    \E[R_T] \le 2 \cD \sqrt{T} \left( \begin{aligned} &\tfrac{\left\lceil V_\G^{2/3} T^{1/3} \kappa^{-2/3} \right\rceil \log T}{\eta} + 4 \kappa^{5/3} V_\G^{1/3} T^{2/3} \\
    &+ ( V_\G^{2/3} T^{4/3} \kappa^{-2/3} + T )^{1/2} \Big( \kappa^2 \log \left(e + e \kappa^2 ( V_\G^{2/3} T^{4/3} \kappa^{-2/3} + T )^{1/2} \right) + \max_{t\in\llbracket T \rrbracket} \lVert g^*_t \rVert^2_\G \Big) \end{aligned} \right)^{1/2} ,
\end{equation*}
and if \( \lambda = V_0^{-1/2} T^{1/2} \), we have 
\begin{equation*}
    \E[R_T] \le 2 \cD \sqrt{T} \sqrt{\dfrac{V_0 \log T}{\eta} + \log \left( e + e \kappa^2 \sqrt{T V_0} \right) \kappa^2 \sqrt{T V_0} + \sqrt{T V_0} \max_{i\in \llbracket V_0 \rrbracket} \lVert \bar g_{t_i:t_{i+1}} \rVert^2_\G} \,.
\end{equation*}

\begin{proof}
    \textbf{Step 1: Controlling the regret of \((f_t)_t\) by the regret of \((\hat g_t)_t\).}
    From Lemma \ref{theorem:CI expectancy}, we have
    \begin{equation*}
        \E \left[ \sum_{t=1}^T \Delta(f_t(x_t), y_t) - \Delta(f^*_t(x_t), y_t) \right] \le 2 \cD \sqrt{T} \sqrt{ \sum_{t=1}^T \E [ \lVert \hat g_t(x_t) - \E [\varphi_t|x_t] \rVert^2_\calH ] } \,.
    \end{equation*}
    We add and subtract a term and then follow the proof of Theorem \ref{theorem:expectancy regret bound},
    \begin{align*}
        &\E \left[ \sum_{t=1}^T \Delta(f_t(x_t), y_t) - \Delta(f^*_t(x_t), y_t) \right] \\
        &\le 2 \cD \sqrt{T} \sqrt{ \sum_{t=1}^T \E [ \lVert \hat g_t(x_t) - \E [\varphi_t|x_t] \rVert^2_\calH - \lVert g^*_t(x_t) - \E [\varphi_t|x_t] \rVert^2_\calH + \lVert g^*_t(x_t) - \E [\varphi_t|x_t] \rVert^2_\calH ] } \,.
    \end{align*}
    We now apply Proposition \ref{theorem:dynamic regret vawk} to obtain
    \begin{equation}
        \E [ R_T ] \le 2 \cD \sqrt{T} \sqrt{ R_T(\lambda, m) + \E \left[ \sum_{t=1}^T \lVert g^*_t(x_t) - \E[\varphi_t | x_t] \rVert^2_\calH \right] } 
    \end{equation}
    where \(R_T(\lambda, m)\) is as in Proposition \ref{theorem:dynamic regret vawk} and is defined as
    \begin{equation*}
        R_T(\lambda, m) = \tfrac{m \log T}{\eta} + \log \left( e + \tfrac{e \kappa^2 T}{\lambda} \right) \sum_{i=1}^m \deff(\lambda, t_{i+1}-t_i) + \lambda m \max_{t\in\llbracket T \rrbracket} \lVert g^*_t \rVert^2_\G + \tfrac{ 4 \kappa V_\G T}{m} .
    \end{equation*}

    \textbf{Case 1: Continuous variations.}
    We bound the effective dimension by \( \deff(\lambda, t_{i+1} - t_i) \le \tfrac{\kappa^2 (t_{i+1} - t_i)}{\lambda} \). Thus we can bound the sum by \( \sum_{i=1}^m \deff(\lambda, t_{i+1}-t_i) \le \kappa^2 T / \lambda \) and obtain
    \begin{equation*}
        R_T(\lambda, m) \le \tfrac{m \log T}{\eta} + \log \left( e + \tfrac{e \kappa^2 T}{\lambda} \right) \tfrac{\kappa^2 T}{\lambda} + \lambda m \max_{t\in\llbracket T \rrbracket} \lVert g^*_t \rVert^2_\G + \tfrac{ 4 \kappa V_\G T}{m} .
    \end{equation*}
    We choose \(\lambda = \sqrt{T/m}\) and get
    \begin{equation*}
        R_T(\lambda, m) \le \tfrac{m \log T}{\eta} + \log ( e + e \kappa^2 \sqrt{Tm} ) \kappa^2 \sqrt{Tm} + \sqrt{Tm} \max_{t\in\llbracket T \rrbracket} \lVert g^*_t \rVert^2_\G + \tfrac{ 4 \kappa V_\G T}{m} .
    \end{equation*}
    We choose \(m = \left\lceil V_\G^{2/3} T^{1/3} \kappa^{-2/3} \right\rceil\) and get
    \begin{multline*}
        R_T(\lambda, m) \le \tfrac{\left\lceil V_\G^{2/3} T^{1/3} \kappa^{-2/3} \right\rceil \log T}{\eta} + 4 \kappa^{5/3} V_\G^{1/3} T^{2/3}  \\
        + ( V_\G^{2/3} T^{4/3} \kappa^{-2/3} + T )^{1/2} \left( \kappa^2 \log (e + e \kappa^2 ( V_\G^{2/3} T^{4/3} \kappa^{-2/3} + T )^{1/2}) + \max_{t\in\llbracket T \rrbracket} \lVert g^*_t \rVert^2_\G \right) .
    \end{multline*}
    We conclude by noting that \(V_\G = \|g_1^*\|_\G + \sum_{t=2}^T \lVert g_t^* - g^*_{t-1} \rVert_\G \le (2T-1) \sup \lVert g \rVert_\G \).

    \textbf{Case 2: Discrete variations.} In the case of discrete distributions data variations there is no more need to approximate the data distributions \((g^*_t)_t\) by the hypothetical forecasters \((\bar g_{t_i:t_{i+1}})_i\). Thus the third term of the sum in Eq. \eqref{eq:dynamic regret kaar sum} is not necessary. And we can replace \(R_T(\lambda, m)\) by \(R^0_T(\lambda, V_0)\) defined as
    \begin{equation*}
        R^0_T(\lambda, V_0) := \dfrac{V_0 \log T}{\eta} + \log \left( e + \tfrac{e \kappa^2 T}{\lambda} \right) \sum_{i=1}^{V_0} \deff(\lambda, t_{i+1} - t_i) + \lambda \sum_{i=1}^{V_0} \lVert \bar g_{t_i:t_{i+1}} \rVert^2_\G .
    \end{equation*}
    We bound the effective dimension by \( \deff(\lambda, t_{i+1} - t_i) \le \tfrac{\kappa^2 (t_{i+1} - t_i)}{\lambda} \). Thus we obtain \( \sum_{i=1}^{V_0} \deff(\lambda, t_{i+1}-t_i) \le \kappa^2 T / \lambda \) and
    \begin{equation}
        R^0_T(\lambda, V_0) \le \dfrac{V_0 \log T}{\eta} + \log \left( e + \tfrac{e \kappa^2 T}{\lambda} \right) \tfrac{\kappa^2 T}{\lambda} + \lambda V_0 \max_{i\in \llbracket V_0 \rrbracket} \lVert \bar g_{t_i:t_{i+1}} \rVert^2_\G .
    \end{equation}
    By choosing \( \lambda = \sqrt{T / V_0} \), we get
    \begin{equation}
        R^0_T(\lambda, V_0) \le \dfrac{V_0 \log T}{\eta} + \log \left( e + e \kappa^2 \sqrt{T V_0} \right) \kappa^2 \sqrt{T V_0} + \sqrt{T V_0} \max_{i\in \llbracket V_0 \rrbracket} \lVert \bar g_{t_i:t_{i+1}} \rVert^2_\G .
    \end{equation}
    Finally we bound the expected regret by
    \begin{equation*}
        \E[R_T] \le 2 \cD \sqrt{T} \sqrt{\dfrac{V_0 \log T}{\eta} + \log \left( e + e \kappa^2 \sqrt{T V_0} \right) \kappa^2 \sqrt{T V_0} + \sqrt{T V_0} \max_{i\in \llbracket V_0 \rrbracket} \lVert \bar g_{t_i:t_{i+1}} \rVert^2_\G} \,.
    \end{equation*}
    We now note that \(V_0\le T\) to conclude the proof.
\end{proof}

\subsection{Refined Regret Bounds}\label{section:appendix capacity condition}
\paragraph{Capacity Condition}
There exists \( \beta \in [0,1] \) and \( Q>0 \) for which
\begin{equation}
    \deff(\lambda) \le Q \left(\dfrac{T}{\lambda}\right)^{\beta} \quad, \forall \lambda>0 .
\end{equation}
When the kernel is bounded the condition above is always satisfied for \(\beta=1\). Indeed we can always bound the effective dimension by \( \deff(\lambda) \le \tfrac{\kappa^2 T}{\lambda} \). Moreover if the eigenvalues of the covariance operator \(C\) decay polynomially \( \sigma_i(C) \le c j^{-\mu} \), for \(c>0, \mu>1\) and \(j\in\mathbb N\), then the capacity condition is satisfied with \(Q=c\) and \( \beta = - 1 / \mu \). Using this bound we may derive a refined bound of the expectation of the dynamic regret.

Assuming that the capacity condition holds, and that there exists \((g^*_t)\) a sequence in \(\G\) such that \( \E \left[ \sum_{t=1}^T \lVert g^*_t(x_t) - \E[\varphi_t | x_t] \rVert^2_\calH \right] = 0 \). Then choosing \( \lambda = T^{\tfrac{\beta}{\beta + 1}} m^{\tfrac{-\beta}{\beta+1}} \) and \( m= \Big \lceil V_\G^{\tfrac{\beta + 1}{\beta +2}} T^{\tfrac{1}{\beta + 2}} \kappa^{\tfrac{-(\beta+1)}{\beta+2}} \Big \rceil \) leads to the following bound on the expected regret
\begin{equation}
    \E [ R_T ] = \tilde O \left( V_\G^{\tfrac{1}{2(\beta+2)}} T^{\tfrac{2\beta +3}{2(\beta+2)}} \right) .
\end{equation}

Indeed using the capacity condition, we bound the effective dimension by \( \deff(\lambda, t_{i+1} - t_i) \le Q \left(\tfrac{t_{i+1} - t_i}{\lambda}\right)^\beta \). Using Jensen's inequality, we then derive
\begin{align*}
    \sum_{i=1}^m \deff(\lambda, t_{i+1} - t_i) &\le Q \sum_{i=1}^m \left(\dfrac{t_{i+1} - t_i}{\lambda}\right)^\beta \\
    &= Q m \sum_{i=1}^m \dfrac{1}{m} \left(\dfrac{t_{i+1} - t_i}{\lambda}\right)^\beta \\
    &\le Q m \left( \sum_{i=1}^m \dfrac{t_{i+1} - t_i}{\lambda m}\right)^\beta \\
    &= Q m^{1-\beta} \left(\dfrac{T}{\lambda}\right)^\beta .
\end{align*}
We obtain
\begin{equation*}
    R_T(\lambda, m) \le \tfrac{m \log T}{\eta} + \log\left( e + \tfrac{e\kappa^2T}{\lambda}\right) Q m^{1-\beta} \left(\tfrac{T}{\lambda}\right)^\beta + \lambda m \max_{t \in \llbracket T \rrbracket} \lVert g^*_t \rVert^2_\G + \tfrac{4 \kappa V_\G T}{m} .
\end{equation*}
We choose \( \lambda = T^{\tfrac{\beta}{\beta + 1}} m^{\tfrac{-\beta}{\beta+1}} \), and obtain
\begin{equation*}
    R_T(\lambda, m) \le \tfrac{m \log T}{\eta} + m^{\tfrac{1}{\beta+1}} T^{\tfrac{\beta}{\beta+1}} \left( Q \log\left( e + e \kappa^2 m^{\tfrac{\beta}{\beta + 1}} T^{\tfrac{1}{\beta+1}} \right) + \max_{t \in \llbracket T \rrbracket} \lVert g^*_t \rVert^2_\G \right) + \tfrac{4 \kappa V_\G T}{m} .
\end{equation*}
We choose \( m= \Big \lceil V_\G^{\tfrac{\beta + 1}{\beta +2}} T^{\tfrac{1}{\beta + 2}} \kappa^{\tfrac{-(\beta+1)}{\beta+2}} \Big \rceil \) and obtain
\begin{multline*}
    R_T(\lambda, m) \le \tfrac{\Big \lceil V_\G^{\tfrac{\beta + 1}{\beta +2}} T^{\tfrac{1}{\beta + 2}} \kappa^{\tfrac{-(\beta+1)}{\beta+2}} \Big \rceil \log T}{\eta} + 4 \kappa^{\tfrac{2\beta +3}{\beta+2}} V_\G^{\tfrac{1}{\beta+2}} T^{\tfrac{\beta+1}{\beta+2}} \\
    + \Big( V_\G^{\tfrac{1}{\beta+2}} T^{\tfrac{\beta+1}{\beta+2}} \kappa^{\tfrac{-1}{\beta+2}} + T^{\tfrac{\beta}{\beta+1}} \Big) \left( Q \log\Big( e + e \kappa^{\tfrac{2\beta+3}{\beta+2}} V_\G^{\tfrac{\beta}{\beta + 2}} T^{\tfrac{2}{\beta+2}} + e T^{\tfrac{1}{\beta+1}} \Big) + \max_{t \in \llbracket T \rrbracket} \lVert g^*_t \rVert^2_\G \right) .
\end{multline*}
We note that the variation \(V_\G = \lVert g^*_1 \rVert_\G + \sum_{t=2}^T \lVert g^*_t - g^*_{t-1} \rVert_\G = O(T) \) to conclude.

\paragraph{Gaussian kernel}
Moreover, in the case of the Gaussian kernel we can bound the effective dimension by \( \smash{\deff(\lambda) \le \left( \log\left( \tfrac{T}{\lambda} \right) \right)^d} \) \citep{altschuler2019massively} where \(d\) is the dimension of the input space \(\X\), to obtain a smaller regret. By choosing \( m = \smash{\sqrt{V_\G T/(\lambda + 1)}} \), we obtain
\begin{equation}
    \E [ R_T ] = \tilde O \left( T^{3/4} V_\G^{1/4} (\lambda+1)^{1/4} \right) .
\end{equation}
In this particular case, if we choose a constant \(\lambda\), we retrieve the power \( T^{3/4} \) from the stationary case.

\section{Dealing with Non-Stationary Data in High Probability}\label{section:appendix non stationary high probability}

In this section we deal with non-stationary data distributions as in Section \ref{section:dynamic regret}, however we bound the cumulative risk in high probability instead of bounding the expected regret. As in Appendix \ref{section:appendix stochastic regret bounds}, we define the feature predictor on a shifted version of the losses in order to apply Theorem 1 of \cite{vanderhoeven2023highprobability}, see Algorithm \ref{algo:non stationary aggregation high proba}. We define the loss \(\ell_t\) as
\begin{equation}
    \ell_t(g) = \left\lVert \varphi_s - \tfrac{1}{2} \hat g_t(x_t) - \tfrac{1}{2} g(x_t) \right\rVert^2_\calH .
\end{equation}
We recall the definition of the filter \(\F_{t-1} = (x_1, y_1, \dots, x_{t-1}, y_{t-1}, x_t)\) and the notation \(\E_t [.]\) that stands for \( \E_{y_t} [. | \F_{t-1}] \).

\begin{algorithm}[H]
\caption{\textit{SALAMI -- Structured prediction ALgorithm with Aggregating MIxture} -- for the high probability setting} 
\label{algo:non stationary aggregation high proba}
\KwIn{\(\lambda>0\), exp-concavity constant \(\eta\) of  \((\ell_t)_t\), kernel \( k : \X \times \X \to \R \)}
\For{Each time step \(t\) in \(1 \dots T\)}{
    Get information \(x_t\in\X\)\\
    \For{Each expert \(s\) in \(1 \dots t\)}{
        Compute \( \hat g_{s:t} = \arg\min_{g\in\G} \sum_{s=s_1}^{t-1} \left\lVert \varphi_s - \tfrac{1}{2} \hat g_s(x_s) - \tfrac{1}{2} g(x_s) \right\rVert_\calH^2 + \lambda \lVert g \rVert_\G^2 + \tfrac{1}{4} \lVert g(x_t) \rVert^2_\calH \)
    }
    \For{Each expert \(s\) in \(1 \dots T\)}{
        Compute the auxiliary loss \( \tilde \ell_t(s) = 
        \begin{cases}
            \ell_t(\hat g_{s:t}) \text{ if } s\le t \\
            \ell_t(\hat g_t) \text{ if } s> t
        \end{cases}  \) \\
        Compute the probability using EWA \( w_t(s) \propto w_{t-1}(s) \exp(-\eta \tilde \ell_t(s)) \) \\
        Compute the probability \( p_t(s) \propto
        \begin{cases}
            w_t(s) \text{ if } s\le t \\
            0 \text{ if } s> t
        \end{cases} \)
    }
    Compute the aggregate predictor \( \hat g_t = \sum_{s=1}^T p_t(s) \hat g_{s:t} \) \\
    Compute the prediction \( \hat z_t = \hat f_t(x_t) = \arg\min_{z\in\Z} \langle \psi(z), \hat g_t(x_t) \rangle_\calH \) \\
    Observe ground truth \(y_t\in\Y\)\\
    Get loss \(\Delta(\hat{z_t}, y_t)\in\R\)
}
\end{algorithm}

The only difference with Appendix \ref{section:appendix non stationary} is the definition of the experts. We shift the losses by \( \tfrac{1}{2} \hat g_s(x_s) \),
\begin{equation}\label{eq:hat gt restart high proba}
    \hat g_{s_1:t} := \arg\min_{g\in\G} \sum_{s=s_1}^{t-1} \left\lVert \varphi_s - \tfrac{1}{2} \hat g_s(x_s) - \tfrac{1}{2} g(x_s) \right\rVert_\calH^2 + \lambda \lVert g \rVert_\G^2 + \tfrac{1}{4} \lVert g(x_t) \rVert^2_\calH
\end{equation}
where \(\hat g_s\) are the aggregate functions defined as in Appendix \ref{section:appendix non stationary}, see Algorithm \ref{algo:non stationary aggregation high proba}. The predictor \(\hat f_t\) and the prediction \(\hat z_t\) are computed as an optimisation problem in function of \(\hat g_t(x_t)\) 
\begin{equation}
    \hat z_t = \hat f_t(x_t) = \arg\min_{z\in\Z} \langle \psi(z), \hat g_t(x_t) \rangle_\calH .
\end{equation}

\paragraph{Analysis}
We now analyse the regret of this algorithm.

We introduce the following technical lemma, which bounds the cumulative risk of the feature predictors by the regret in high probability using Theorem 1 of \cite{vanderhoeven2023highprobability}.
\begin{restatable}[Dynamic Cumulative Risk of the Feature Predictors]{lmm}{}
\label{theorem:Dynamic Cumulative Risk of the Feature Predictors}
    Let \( (\hat g_t)_t \) be defined as in Algorithm \ref{algo:non stationary aggregation high proba}, \(m\in\mathbb N\) be defined as in \eqref{eq:definition m}, any sequence \( (g^*_t)_t \in \G^T \) and \( \delta \in (0,1] \). Let \( R_T(\lambda, m) \) be defined as follows 
    \begin{equation}
        \sum_{t=1}^T \lVert \hat g_t(x_t) - \varphi_t \rVert^2_\calH - \left\lVert \tfrac{1}{2} g^*_t(x_t) + \tfrac{1}{2} \hat g_t(x_t) - \varphi_t \right\rVert^2_\calH \le R_T(\lambda, m) .
    \end{equation}
    Let \( \gamma \) be defined as in Eq.~\eqref{eq:gamma}. Then with probability \( 1-\delta \)
    \begin{equation}
        \sum_{t=1}^T \E_t [ \lVert \hat g_t(x_t) - \varphi_t \rVert^2_\calH - \lVert g^*_t(x_t) - \varphi_t \rVert^2_\calH ] \le 2 R_T(\lambda, m) + 2 \gamma \log(\delta^{-1}) .
    \end{equation}
\end{restatable}

\begin{proof}
    Let us recall the notations of \cite{vanderhoeven2023highprobability} in order to apply the Theorem 1 of their paper.
    Let \(\tilde \ell_t\) denote the shifted loss as in \cite{vanderhoeven2023highprobability}
    \begin{equation}
        \tilde \ell_t(g) = \ell_t\left(\tfrac{1}{2} g + \tfrac{1}{2} \hat g_t\right) = \left\lVert \tfrac{1}{2} g(x_t) + \tfrac{1}{2} \hat g_t(x_t) - \varphi_t \right\rVert^2_\calH .
    \end{equation}
    We defined \(R_T(\lambda, m)\) such that
    \begin{equation*}
        \sum_{t=1}^T \tilde \ell_t (\hat g_t(x_t)) - \tilde \ell_t(g^*_t) \le R_T(\lambda, m)
    \end{equation*}
    with the only assumption on \((g^*_t)_t\) that they are functions in the space \(\G\). Using the convexity of \(\tilde \ell\), we use Jensen inequality
    \begin{equation*}
        \tilde \ell \left( \E_{g\sim Q_t}[g] \right) \le \E_{g\sim Q_t} [\tilde \ell (g)] ,
    \end{equation*}
    where \((Q_t)_t\) are some distributions over \(\G\), and by convexity of the space \(\G\) we have that \( \E_{g\sim Q_t}[g] \in \G \). We derive the following inequality
    \begin{equation*}
        \sum_{t=1}^T \tilde \ell_t(\hat g_t) - \E_{g\sim Q_t} [\tilde \ell_t(g)] \le \sum_{t=1}^T \tilde \ell_t(\hat g_t) - \tilde \ell \left( \E_{g\sim Q_t}[g] \right) \le R_T(\lambda, m).
    \end{equation*}
    We now remark that the proof of the Theorem 1 of \cite{vanderhoeven2023highprobability} can be applied with a non-stationary baseline \((Q_t)_t\) instead of \(Q\). This concludes the proof.
\end{proof}

We use this lemma and a comparison inequality to bound the cumulative risk of the predictors.

\begin{restatable}[Expected Regret in a Non-Stationary Environment]{thm}{}
     Assume that there exists $(g_t^*)$ a sequence in $\G$ such that \( \smash{\E \big[ \sum_{t=1}^T \lVert g_t^*(x_t) - \E[\varphi_t | x_t] \rVert^2_\calH \big] = 0} \). Let \(\delta \in (0,1]\). Then, Algorithm~\ref{algo:non stationary aggregation high proba} run with $\lambda >0$ and $\smash{\eta = 1/2(\kappa \sup \lVert g \rVert_\G + 1)^2}$ satisfies with probability \(1-\delta\)
    \begin{equation}
        \E [ R_T ] = 
        \left\{
        \begin{array}{ll}
         \tilde O \left( V_\G^{1/6} T^{5/6} + T^{1/2} \sqrt{\log(\delta^{-1})} \right) & \text{if } \lambda = V_\G^{-1/3} T^{1/3} \\
         \tilde O \left( V_0^{1/4} T^{3/4} + T^{1/2} \sqrt{\log(\delta^{-1})} \right) & \text{if } \lambda = V_0^{-1/2} T^{1/2} 
        \end{array}
        \right. \,.
    \end{equation}
\end{restatable}

Precisely, let \(\gamma\) let be defined as in Eq.~\eqref{eq:gamma}, if \(\lambda = V_\G^{-1/3} T^{1/3}\) we have
\begin{align*}
    &\sum_{t=1}^T \E_t [ \Delta(\hat f_t(x_t), y_t) - \Delta(f^*_t(x_t), y_t) ] \\
    &\le 2 \cD \sqrt{2T} \left(\begin{aligned}
        &\tfrac{\left\lceil V_\G^{2/3} T^{1/3} \kappa^{-2/3} \right\rceil \log T}{\eta} + \left(1 + \kappa \sup \lVert g \rVert_\G \right) \kappa^{5/3} V_\G^{1/3} T^{2/3} \\
        &+ ( V_\G^{2/3} T^{4/3} \kappa^{-2/3} + T )^{1/2} \tfrac{\kappa^2 B^2}{16} \log \left(e + \tfrac{e \kappa^2}{4} ( V_\G^{2/3} T^{4/3} \kappa^{-2/3} + T )^{1/2}\right) \\
        &+ ( V_\G^{2/3} T^{4/3} \kappa^{-2/3} + T )^{1/2} \max_{t\in\llbracket T \rrbracket} \lVert g^*_t \rVert^2_\G + \gamma \log(\delta^{-1})
    \end{aligned} \right)^{1/2} ,
\end{align*}
and if \( \lambda = V_0^{-1/2} T^{1/2} \) we have
\begin{multline*}
    \sum_{t=1}^T \E_t [ \Delta(\hat f_t(x_t), y_t) - \Delta(f^*_t(x_t), y_t) ] \\
    \le 2 \cD \sqrt{2T} \Big( \tfrac{V_0 \log T}{\eta} + \tfrac{B^2 \kappa^2 \sqrt{T V_0}}{16} \log \Big( e + \tfrac{e \kappa^2 \sqrt{T V_0}}{4} \Big) +  \sqrt{T V_0} \max_{i} \lVert \bar g_{t_i:t_{i+1}} \rVert^2 + \gamma \log(\delta^{-1}) \Big)^{1/2} \,.
\end{multline*}

\begin{proof}
    \textbf{Step 1: Controlling the regret of \((f_t)_t\) by the regret of \((\hat g_t)_t\).}
    From Lemma \ref{theorem:CI cumulative risk}, we have
    \begin{align*}
        \sum_{t=1}^T &\E_t [\Delta(\hat f_t (x_t), y_t) - \Delta(f^*_t(x_t), y_t) ]\\
        &\le 2 \cD \sqrt{T} \sqrt{ \sum_{t=1}^T \E_t [ \lVert \hat g_t(x_t) - \E [\varphi_t | x_t] \rVert^2_\calH ] } \\
        &= 2 \cD \sqrt{T} \sqrt{ \sum_{t=1}^T \E_t [ \lVert \hat g_t(x_t) - \E [\varphi_t | x_t] \rVert^2_\calH - \lVert g^*_t(x_t) - \E [\varphi_t | x_t] \rVert^2_\calH + \lVert g^*_t(x_t) - \E [\varphi_t | x_t] \rVert^2_\calH ] } \\
        &= 2 \cD \sqrt{T} \sqrt{ \sum_{t=1}^T \E_t [ \lVert \hat g_t(x_t) -\varphi_t \rVert^2_\calH - \lVert g^*_t(x_t) - \varphi_t \rVert^2_\calH ] + \E_t [ \lVert g^*_t(x_t) - \E [\varphi_t | x_t] \rVert^2_\calH ] } .
    \end{align*}

    We may now apply Lemma \ref{theorem:Dynamic Cumulative Risk of the Feature Predictors} to obtain
    \begin{align}
        \sum_{t=1}^T &\E_t [ \Delta(\hat f_t(x_t), y_t) - \Delta(f^*_t(x_t), y_t) ] \nonumber \\
        &\le 2 \cD \sqrt{T} \sqrt{2 R_T(\lambda, m) + 2 \gamma \log(\delta^{-1}) + \sum_{t=1}^T \lVert g^*_t(x_t) - \E[\varphi_t|x_t] \rVert^2_\calH } 
    \end{align}
    where \(R_T(\lambda, m)\) is defined as 
    \begin{equation*}
        R_T(\lambda, m) = \tfrac{m \log T}{\eta} + \tfrac{B^2}{4} \log \left( e + \tfrac{e \kappa^2 T}{4\lambda} \right) \sum_{i=1}^m \deff(4\lambda, t_{i+1}-t_i) + \lambda m \max_{t\in\llbracket T \rrbracket} \lVert g^*_t \rVert^2_\G + \tfrac{\left(1 + \kappa \sup \lVert g \rVert_\G \right) \kappa V_\G T}{m}
    \end{equation*}
    and comes from Proposition \ref{theorem:dynamic regret vawk} for shifted losses.

    \textbf{Case 1: Continuous variations.}
    We bound the effective dimension by \( \deff(4\lambda, t_{i+1} - t_i) \le \tfrac{\kappa^2 (t_{i+1} - t_i)}{4\lambda} \). Thus we can bound the sum by \( \sum_{i=1}^m \deff(4\lambda, t_{i+1}-t_i) \le \kappa^2 T / 4\lambda \) and obtain
    \begin{equation*}
        R_T(\lambda, m) \le \tfrac{m \log T}{\eta} + \tfrac{B^2}{4} \log \left( e + \tfrac{e \kappa^2 T}{4\lambda} \right) \tfrac{\kappa^2 T}{4\lambda} + \lambda m \max_{t\in\llbracket T \rrbracket} \lVert g^*_t \rVert^2_\G + \tfrac{ \left(1 + \kappa \sup \lVert g \rVert_\G \right) \kappa V_\G T}{m} .
    \end{equation*}
    We choose \(\lambda = \sqrt{T/m}\) and get
    \begin{equation*}
        R_T(\lambda, m) \le \tfrac{m \log T}{\eta} + \tfrac{B^2}{4} \log \left( e + \tfrac{e \kappa^2 \sqrt{Tm}}{4} \right) \tfrac{\kappa^2 \sqrt{Tm}}{4} + \sqrt{Tm} \max_{t\in\llbracket T \rrbracket} \lVert g^*_t \rVert^2_\G + \tfrac{ \left(1 + \kappa \sup \lVert g \rVert_\G \right) \kappa V_\G T}{m} .
    \end{equation*}
    We choose \(m = \left\lceil V_\G^{2/3} T^{1/3} \kappa^{-2/3} \right\rceil\) and get
    \begin{multline*}
        R_T(\lambda, m) \le \tfrac{\left\lceil V_\G^{2/3} T^{1/3} \kappa^{-2/3} \right\rceil \log T}{\eta} + \left(1 + \kappa \sup \lVert g \rVert_\G \right) \kappa^{5/3} V_\G^{1/3} T^{2/3} \\
        + ( V_\G^{2/3} T^{4/3} \kappa^{-2/3} + T )^{1/2} \left( \tfrac{\kappa^2 B^2}{16} \log \left(e + \tfrac{e \kappa^2}{4} ( V_\G^{2/3} T^{4/3} \kappa^{-2/3} + T )^{1/2}\right) + \max_{t\in\llbracket T \rrbracket} \lVert g^*_t \rVert^2_\G \right)  .
    \end{multline*}
    We conclude by noting that \(V_\G \le (2T-1) \sup \lVert g \rVert_\G\).

    \textbf{Case 2: Discrete variations.} In the case of discrete distributions data variations there is no more need to approximate the data distributions \((g^*_t)_t\) by the hypothetical forecasters \((\bar g_{t_i:t_{i+1}})_i\). Thus the third term of the sum in Eq. \eqref{eq:dynamic regret kaar sum} is not necessary. And we can replace \(R_T(\lambda, m)\) by \(R^0_T(\lambda, V_0)\) defined as
    \begin{equation*}
        R^0_T(\lambda, V_0) := \dfrac{V_0 \log T}{\eta} + \tfrac{B^2}{4} \log \left( e + \tfrac{e \kappa^2 T}{4\lambda} \right) \sum_{i=1}^{V_0} \deff(4\lambda, t_{i+1}-t_i) + \lambda \sum_{i=1}^{V_0} \lVert \bar g_{t_i:t_{i+1}} \rVert^2_\G .
    \end{equation*}
    We bound the effective dimension by \( \deff(4\lambda, t_{i+1} - t_i) \le \tfrac{\kappa^2 (t_{i+1} - t_i)}{4\lambda} \). Thus we bound the sum by \( \sum_{i=1}^{V_0} \deff(4\lambda, t_{i+1}-t_i) \le \kappa^2 T / 4\lambda \) and obtain
    \begin{equation}
        R^0_T(\lambda, V_0) \le \tfrac{V_0 \log T}{\eta} + \tfrac{B^2 \kappa^2 T}{16 \lambda} \log \left( e + \tfrac{e \kappa^2 T}{4\lambda} \right) + \lambda V_0 \max_{i\in \llbracket V_0 \rrbracket} \lVert \bar g_{t_i:t_{i+1}} \rVert^2_\G .
    \end{equation}
    By choosing \( \lambda = \sqrt{T / V_0} \), we get
    \begin{equation}
        R^0_T(\lambda, V_0) \le \tfrac{V_0 \log T}{\eta} + \tfrac{B^2 \kappa^2 \sqrt{T V_0}}{16} \log \left( e + \tfrac{e \kappa^2 \sqrt{T V_0}}{4} \right) +  \sqrt{T V_0} \max_{i\in \llbracket V_0 \rrbracket} \lVert \bar g_{t_i:t_{i+1}} \rVert^2_\G .
    \end{equation}
    We then bound the cumulative risk using Step 1.
    \begin{multline*}
        \sum_{t=1}^T \E_t [ \Delta(\hat f_t(x_t), y_t) - \Delta(f^*_t(x_t), y_t) ] \\
        \le 2 \cD \sqrt{2T} \sqrt{\tfrac{V_0 \log T}{\eta} + \tfrac{B^2 \kappa^2 \sqrt{T V_0}}{16} \log \left( e + \tfrac{e \kappa^2 \sqrt{T V_0}}{4} \right) +  \sqrt{T V_0} \max_{i\in \llbracket V_0 \rrbracket} \lVert \bar g_{t_i:t_{i+1}} \rVert^2_\G + \gamma \log(\delta^{-1})}
    \end{multline*}
    We now note that \(V_0\le T\) to conclude the proof.
\end{proof}

\subsection{Refined Regret Bounds}

In this section we assume that the capacity condition holds and we derive refined bounds of the cumulative risk. For more details about the capacity condition see Appendix \ref{section:appendix capacity condition}.
Assuming that there exists \((g^*_t)\) a sequence in \(\G\) such that \( \E \left[ \sum_{t=1}^T \lVert g^*_t(x_t) - \E[\varphi_t | x_t] \rVert^2_\calH \right] = 0 \). Let \(\delta \in (0,1]\). Then choosing \( \lambda = T^{\tfrac{\beta}{\beta + 1}} m^{\tfrac{-\beta}{\beta+1}} \) and \( m= \Big \lceil V_\G^{\tfrac{\beta + 1}{\beta +2}} T^{\tfrac{1}{\beta + 2}} \kappa^{\tfrac{-(\beta+1)}{\beta+2}} \Big \rceil \) leads to the following bound on the expected regret with probability \(1-\delta\)
\begin{equation}
    \sum_{t=1}^T \E_t [ \Delta(\hat f_t(x_t), y_t) - \Delta(f^*_t(x_t), y_t) ] = \tilde O \left( V_\G^{\tfrac{1}{2(\beta+2)}} T^{\tfrac{2\beta +3}{2(\beta+2)}} + T^{1/2} \sqrt{\log(\delta^{-1})} \right) .
\end{equation}

Indeed using the capacity condition, we bound the effective dimension by \( \deff(4\lambda, t_{i+1} - t_i) \le Q \left(\tfrac{t_{i+1} - t_i}{4\lambda}\right)^\beta \). Using Jensen's inequality, we then derive
\begin{align*}
    \sum_{i=1}^m \deff(4\lambda, t_{i+1} - t_i) &\le Q \sum_{i=1}^m \left(\dfrac{t_{i+1} - t_i}{4\lambda}\right)^\beta \\
    &= Q m \sum_{i=1}^m \dfrac{1}{m} \left(\dfrac{t_{i+1} - t_i}{4\lambda}\right)^\beta \\
    &\le Q m \left( \sum_{i=1}^m \dfrac{t_{i+1} - t_i}{4\lambda m}\right)^\beta \\
    &= Q m^{1-\beta} \left(\dfrac{T}{4\lambda}\right)^\beta .
\end{align*}
We obtain
\begin{equation*}
        R_T(\lambda, m) \le \tfrac{m \log T}{\eta} + \tfrac{Q B^2}{4} m^{1-\beta} \left(\tfrac{T}{4\lambda}\right)^\beta \log \left( e + \tfrac{e \kappa^2 T}{4\lambda} \right) + \lambda m \max_{t\in\llbracket T \rrbracket} \lVert g^*_t \rVert^2_\G + \tfrac{ \left(1 + \kappa \sup \lVert g \rVert_\G \right) \kappa V_\G T}{m} .
    \end{equation*}
We choose \( \lambda = T^{\tfrac{\beta}{\beta + 1}} m^{\tfrac{-\beta}{\beta+1}} \), and obtain
\begin{equation*}
    R_T(\lambda, m) \le \tfrac{m \log T}{\eta} + \tfrac{ \left(1 + \kappa \sup \lVert g \rVert_\G \right) \kappa V_\G T}{m} + m^{\tfrac{1}{\beta+1}} T^{\tfrac{\beta}{\beta+1}} \left( \tfrac{Q B^2 4^{-\beta}}{4} \log\left( e + \tfrac{e \kappa^2}{4} m^{\tfrac{\beta}{\beta + 1}} T^{\tfrac{1}{\beta+1}} \right) + \max_{t\in\llbracket T \rrbracket} \lVert g^*_t \rVert^2_\G \right)  .
\end{equation*}
We choose \( m= \Big \lceil V_\G^{\tfrac{\beta + 1}{\beta +2}} T^{\tfrac{1}{\beta + 2}} \kappa^{\tfrac{-(\beta+1)}{\beta+2}} \Big \rceil \) to obtain
\begin{multline*}
    R_T(\lambda, m) \le \tfrac{\Big \lceil V_\G^{\tfrac{\beta + 1}{\beta +2}} T^{\tfrac{1}{\beta + 2}} \kappa^{\tfrac{-(\beta+1)}{\beta+2}} \Big \rceil \log T}{\eta} + \left(1 + \kappa \sup \lVert g \rVert_\G \right) \kappa^{\tfrac{2\beta +3}{\beta+2}} V_\G^{\tfrac{1}{\beta+2}} T^{\tfrac{\beta+1}{\beta+2}} \\
    + \Big( V_\G^{\tfrac{1}{\beta+2}} T^{\tfrac{\beta+1}{\beta+2}} \kappa^{\tfrac{-1}{\beta+2}} + T^{\tfrac{\beta}{\beta+1}} \Big) \Big( \tfrac{Q B^2 4^{-\beta}}{4} \log\Big( e + e \tfrac{\kappa^{\tfrac{2\beta+3}{\beta+2}} V_\G^{\tfrac{\beta}{\beta + 2}} T^{\tfrac{2}{\beta+2}} + T^{\tfrac{1}{\beta+1}}}{4} \Big) + \max_{t \in \llbracket T \rrbracket} \lVert g^*_t \rVert^2_\G \Big) .
\end{multline*}
We conclude by noting that \(V_\G = \|g_1^*\|_\G + \sum_{t=2}^T \lVert g_t^* - g^*_{t-1} \rVert_\G \le (2T-1) \sup \lVert g \rVert_\G \).

\end{document}